\documentclass[conference]{IEEEtran}
\IEEEoverridecommandlockouts

\usepackage[utf8]{inputenc}
\usepackage[T1]{fontenc}
\usepackage{amsmath,amssymb,amsfonts}
\usepackage{amsthm}
\usepackage{xspace}
\usepackage[dvipsnames,svgnames,x11names]{xcolor}
\usepackage{graphicx}
\usepackage{booktabs}
\usepackage{multirow}
\usepackage{array}
\usepackage{tabularx}
\usepackage[shortlabels]{enumitem}
\usepackage{caption}
\usepackage{subcaption}
\usepackage{float}
\usepackage[expansion=false]{microtype}
\usepackage{placeins}
\usepackage{url}
\usepackage{tikz}
\usetikzlibrary{shapes.geometric,arrows.meta,positioning,
                fit,backgrounds,decorations.pathreplacing}
\usepackage{algorithm}
\usepackage{algpseudocode}
\usepackage{mdframed}
\usepackage[most]{tcolorbox}
\usepackage{cuted}      
\makeatletter
\providecommand\@setmarks{}
\makeatother
\tcbuselibrary{listings,breakable,skins}
\usepackage[numbers,sort&compress]{natbib}
\usepackage[hidelinks,colorlinks=true,linkcolor=NavyBlue,
            citecolor=NavyBlue,urlcolor=NavyBlue]{hyperref}
\usepackage{xcolor}
\usepackage{hyperref}

\theoremstyle{definition}
\newtheorem{definition}{Definition}
\newtheorem{proposition}[definition]{Proposition}
\theoremstyle{remark}

\newtcolorbox{promptbox}{
  colback=gray!5,
  colframe=black,
  boxrule=0.5pt,
  arc=3pt,
  left=4pt, right=4pt, top=4pt, bottom=4pt,
  fontupper=\scriptsize\ttfamily,
  breakable,
}

\newmdenv[
  linecolor=gray!40,
  backgroundcolor=gray!6,
  roundcorner=4pt,
  innertopmargin=8pt,
  innerbottommargin=8pt,
  innerleftmargin=10pt,
  innerrightmargin=10pt,
  skipabove=6pt,
  skipbelow=6pt,
]{placeholder}

\newcommand{\order}{\textsc{Order}\xspace}
\newcommand{\orderworld}{\textsc{Order-World}\xspace}
\newcommand{\orderbench}{\textsc{Order-Bench}\xspace}
\newcommand{\orderspatial}{\textsc{Order-Spatial}\xspace}
\newcommand{\dkl}{\textsc{DKL}\xspace}

\newcommand{\klora}{K-LoRA\xspace}
\newcommand{\slora}{S-LoRA\xspace}
\newcommand{\svr}{\textsc{svr}\xspace}
\newcommand{\sdhr}{\textsc{sdhr}\xspace}
\newcommand{\parcmd}{\textsc{par}\xspace}

\newcommand{\poset}{\mathcal{P}}
\newcommand{\objects}{\mathcal{O}}
\newcommand{\config}{\mathcal{C}}
\newcommand{\corpus}{\mathcal{D}}
\newcommand{\score}{f_{\mathrm{crit}}}
\newcommand{\HA}{H_{\mathrm{A}}}
\newcommand{\HB}{H_{\mathrm{B}}}
\newcommand{\tauorder}{\ensuremath{\tau_{\mathrm{order}}}\xspace}
\newcommand{\tauorders}{\ensuremath{\tau^{s}_{\mathrm{order}}}\xspace}
\newcommand{\tauorderu}{\ensuremath{\tau^{u}_{\mathrm{order}}}\xspace}

\begin{document}

\title{\order: A Fictitious-World Benchmark for Domain-Adaptive Embodied AI}

\author{\IEEEauthorblockN{%
Sai Krishna Reddy Sathi\,
\href{https://orcid.org/0009-0005-4136-3005}{%
\textcolor{OliveGreen}{\textsuperscript{\textbf{iD}}}}
\quad
Anuj Tiwari\,
\href{https://orcid.org/0000-0003-3685-8888}{%
\textcolor{OliveGreen}{\textsuperscript{\textbf{iD}}}}
}
\IEEEauthorblockA{Indian Institute of Technology Madras, India\\
\texttt{me21b181@smail.iitm.ac.in}}}

\maketitle

\begin{abstract}
Adapting language models to new domains via continual pre-training raises a
basic evaluation problem: if the training corpus overlaps with what the model
already knows, performance gains cannot be cleanly attributed to new learning
rather than pre-existing knowledge. This matters most for knowledge-intensive,
task-light (KHTL) robot deployments -- pharmaceutical dispensing, hazardous-material
handling, facility-specific protocols, where the physical task is simple but
the governing rules are proprietary and safety-critical, and where extensive
live testing is costly or unsafe. We introduce \order (\textbf{O}ntology-driven \textbf{D}ecision-making for \textbf{E}mbodied
\textbf{R}easoning), a benchmark built on a
fictitious world: a 342{,}069-token synthetic corpus defining a self-consistent
physics that cannot appear in any model's pre-training data. \order pairs a
500-question knowledge test (\orderbench) with a harder compositional task,
\orderspatial: ordering objects for safe manipulation across both familiar and
entirely novel scenes. GPT-4.1 without adaptation scores \emph{below chance} on
\orderspatial ($\tauorder=0.441$), showing its priors actively conflict with
the invented physics. After continual pre-training, small models improve
substantially on both familiar and novel scenes alike evidence of genuine
world-model induction rather than memorization. We then carry this through to a robot pipeline: models that answer the knowledge test well often cannot
produce valid, executable plans without a further skill-adaptation stage,
after which small, fully offline models outperform GPT-4.1 even when GPT-4.1
is given retrieval access to the same rules ($\tauorder=0.848$ vs.\ $0.606$), on a full perception-to-execution loop demonstrated on a simulated
\texttt{iiwa7} arm with human-in-the-loop correction. Throughout,
\orderspatial performance, not knowledge-test accuracy is what predicts
real plan quality. \order will be released publicly as a reusable testbed for evaluating knowledge acquisition in domain-adapted robotic systems.
\end{abstract}


\section{Introduction}
\label{sec:intro}

Robots deployed in high-stakes domains face a paradox: the task can be
physically narrow, but the required knowledge is vast, domain-specific, and
entirely absent from public corpora. A pharmaceutical dispensing robot must
understand drug-sequence constraints from temperature-sensitivity windows; a
nuclear material handler must apply prioritisation rules derived from
classified decay-rate data; a semiconductor fabrication robot must follow
facility-specific contamination sequencing protocols. In every case the
physical task is pick-and-place on a structured workspace, but the governing
knowledge is proprietary. We call this class of deployments
\emph{Knowledge-intensive, Task-light} (KHTL) systems.
Figure~\ref{fig:master} situates the problem and previews the full pipeline
this paper validates.

Adapting a foundation small language model (SLM) as the cognitive module of a
KHTL robot via continual learning is
natural~\citep{ke2022continual,wang2024comprehensive,luo2023empirical}, but
immediately raises an evaluation question existing frameworks cannot answer
cleanly: has the model genuinely internalised the new domain physics, or is it
merely pattern-matching against pre-existing parametric knowledge?
World-model theory~\citep{lecun2022path} frames this precisely: intelligent
systems build understanding over an abstract layer of constructed concepts
that make observations predictable. We use \emph{world-model induction} to
denote exactly this outcome - a model that has internalised the abstract
layer of a domain and therefore applies its constructed concepts to
configurations it has never seen - as opposed to \emph{memorisation}, in
which only the surface form of training instances is retained. A model that
has genuinely learned a KHTL domain has internalised the abstract layer; one
that has merely memorised surface patterns has not. Distinguishing the two
empirically is the purpose of \order.

\subsection{The contamination problem}
If the knowledge-ingestion corpus overlaps with pre-training data, any
benchmark improvement is ambiguous~\citep{deng2024investigating,mallen2023tails,
zhu2023physics31}: it may reflect genuine knowledge acquisition or retrieval of
pre-existing parametric memory. In NLP this motivated synthetic corpora built
around fictitious entities such as synthetic biographies and fabricated
chemical compounds~\citep{zhu2023physics31,allen2023physlm,liu2024fictitious}.
No analogous contamination-free evaluation framework exists for embodied AI,
and \order is designed to fill that gap.

\subsection{\order in brief}
\order takes perceptually grounded, visually familiar primitives and constructs
a fresh abstract physics over them: a complete fictitious world defining how
objects interact, what properties they carry, and how combinations behave. The
corpus is entirely novel (e.g.\ the ``Rubescent Decay Field'' governing red
objects, ``null-reactive tensors'' of white objects), so no frontier model's
pre-training corpus can represent it, and its priors may actively conflict with
it. The \orderspatial seen/unseen split operationalises the classical
systematicity debate~\citep{fodor1988connectionism,lake2018scan}: does
continual pre-training (CPT) induce genuine world-model induction ($\HB$), or
surface memorisation ($\HA$)?

\subsection{From evaluation to deployment}
A benchmark for KHTL robotics is only as valuable as the deployment decisions
it supports. We therefore carry the validated models through the
\emph{operational} question as well: given a model whose domain adaptation has
been shown contamination-free, how do we construct a reliable pipeline from
internalised knowledge to physically deployable plans? Two separations
structure the answer. First, domain knowledge and the skill of Behavior Tree
(BT) generation are separable capabilities: knowledge-ingested (KI) models,
despite strong \orderbench scores, produce schema-valid BTs at rates near zero
without task-specific fine-tuning, motivating skill fine-tuning via \slora on
424 paired (NL-query, gold-BT) examples. Second, injecting a full structured
ontology context into a fine-tuned SLM causes \emph{context hijacking}: the
retrieved chunks re-activate the analytical-prose mode that CPT rewarded and
the model abandons XML output entirely, collapsing schema validity to $0\%$.
A two-stage retrieval pipeline with a condensed priority-reasoning call, and
optionally a constrained-start XML prefix, recovers functional outputs; full
grammar-constrained decoding at every step, by contrast, degrades BT quality.

\begin{strip}
  \centering
  \includegraphics[
    width=0.85\textwidth
  ]{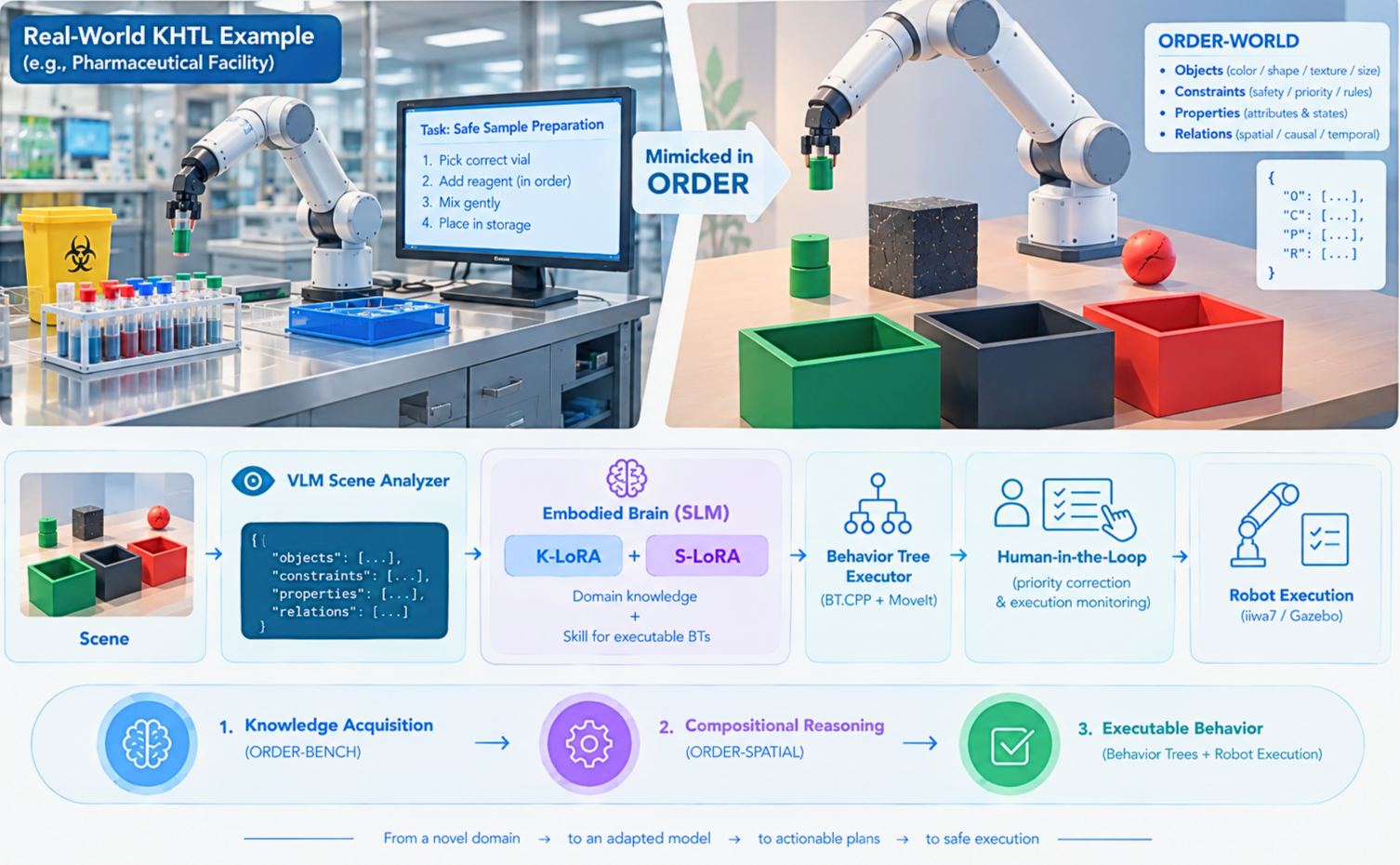}
  \captionof{figure}{%
    \textbf{Why \order, and what it is for.}
      A knowledge-intensive, task-light (\textsc{khtl}) deployment is physically
      simple but epistemically hard: the governing knowledge is proprietary and
      absent from pre-training data, so improvements measured on any
      in-domain benchmark are ambiguous between knowledge acquisition and
      parametric recall. \order replaces the proprietary domain with a
      contamination-free fictitious physics that preserves the
      \emph{structural} properties which make the real problem hard, evaluates
      whether a candidate continual-learning pipeline has genuinely internalised
      it, and then carries the validated model through skill fine-tuning to
      executable, human-correctable Behavior Trees running on a manipulator.
      \\
      \quad{\scriptsize\emph{*Parts of this figure were created with AI assistance for visualization.}}
  }
  \label{fig:master}
\end{strip}

\subsection{Central findings}
RAG nearly saturates all non-spatial \orderbench categories ($\geq\!90\%$),
confirming corpus coherence. Compositional spatial priority ordering remains
persistently hard: GPT-4.1 reaches $86\%$ with retrieval while SLMs stall at
$51$--$72\%$; the bottleneck is reasoning capacity, not retrieval
quality~\citep{yang2026fine,mallen2023not}. On the free-form \orderspatial
task, GPT-4.1 \emph{without} domain adaptation scores $\tauorder=0.441$, below
the $0.500$ random baseline, its pre-training priors conflicting with the
fictitious physics. Parametric CPT internalises the physics into weights,
enabling compositional application at inference without in-context rule
integration~\citep{liu2024lost}. Downstream, closed-book Domain-BT-LMs reach
$\tauorder = 0.848$ against $0.606$ for GPT-4.1 with the same structured
retrieval pipeline, and \orderspatial concordance - not MCQ aggregate -
is what anticipates that ranking.

\subsection{Contributions}
\begin{enumerate}[leftmargin=1.4em,itemsep=1pt,topsep=2pt,label=(\roman*)]
  \item \textbf{\orderworld}: a 342k-token six-layer synthetic corpus with a
    provably contamination-free physics.
  \item \textbf{\orderbench MCQ}: a 500-question Bloom-weighted benchmark with
    formal gold-label uniqueness guarantees.
  \item \textbf{\orderspatial}: a seen/unseen spatial priority task with three
    concordance-based metrics, directly operationalising the $\HA$/$\HB$ axis.
  \item \textbf{RAG dissociation}: retrieval saturates factual recall but
    cannot support compositional priority ordering, establishing parametric
    CPT as necessary for KHTL tasks.
  \item \textbf{World-model induction evidence}: best-adapted SLMs show
    $+0.18$--$+0.26$ \tauorder on \emph{both} seen and unseen splits, with
    seen/unseen gaps $\leq\!0.05$ and 2 of 3 best variants showing an inverted
    gap.
  \item \textbf{A two-stage adaptation pipeline} (\klora $+$ \slora) with
    empirical validation of knowledge--skill disentanglement across three SLM
    families and eight knowledge-ingestion variants.
  \item \textbf{Characterisation of context hijacking} in fine-tuned SLMs, and
    two structured retrieval pipelines that recover functional outputs while
    preserving the model's fine-tuned generation distribution.
  \item \textbf{Empirical closure of the knowledge-to-action chain}: a HITL
    framework deployed on a simulated \texttt{iiwa7} arm, where
    single-correction re-prompting reliably produces compilable, correctly
    ordered trees, and where \orderspatial concordance is confirmed as a
    stronger predictor of BT quality than MCQ aggregate.
\end{enumerate}

\section{Related Work}
\label{sec:related}

\subsection{Continual learning for domain adaptation}
CPT has demonstrated effectiveness in scientific, medical, and financial
domains~\citep{ke2022continual,luo2023empirical,gururangan2020domain,
shi2024continual}. Representing the same knowledge in multiple surface formats
substantially improves downstream extraction~\citep{allen2023physlm}, and
corpus structure is a primary determinant of continual-learning
gains~\citep{jiang2025mixcpt,liu2024scpt}. The contamination problem has been
acknowledged in NLP~\citep{zhu2023physics31,mallen2023tails} but no analogous
framework has been proposed for embodied AI.

\subsection{Opacity of end-to-end policies}
Vision-Language-Action (VLA) models such as RT-2, OpenVLA, and $\pi_0$ achieve
impressive task generalisation by jointly training perception, language, and
control within end-to-end neural
policies~\citep{brohan2023rt2,kim2024openvla,driess2023palme,black2024pi0,
ahn2022saycan}. Yet these models remain fundamentally opaque: they offer no
principled account of why a particular action is selected, making failure
prediction and post-hoc auditing intractable in safety-critical settings.
Recent work on VLA safety has introduced constrained reinforcement learning via
CMDPs to reduce safety violations~\citep{zhang2025safevla}, while mechanistic
steering approaches modulate internal activations at
inference~\citep{haon2025mechanistic}. Despite these advances, current VLAs
remain learned-constraint systems rather than formally verifiable ones,
limiting their applicability where human-readable, modifiable reasoning is
essential~\citep{kawaharazuka2025vision,ma2024vlasurvey}.

\subsection{Structured plan representations: Behavior Trees}
Classical manipulation pipelines rely on Finite State Machines (FSMs),
symbolic planners, and hierarchical controllers. FSM-based
frameworks~\citep{bohren2010smach} offer deterministic execution but suffer
from poor scalability and modularity. Behavior Trees (BTs) address these
limitations through a rooted directed tree in which periodic ticks propagate
from the root, with nodes returning Success, Failure, or
Running~\citep{colledanchise2018bt,iovino2022survey}. Compared to FSMs, BTs
offer improved reactivity, modularity, and graphical readability, with explicit
success--failure conditions that make decision logic transparent to human
operators - precisely the properties required in safety-critical and
collaborative robotics. The \textsc{BehaviorTree.CPP} execution framework
provides a mature runtime for deploying BTs on physical systems.

\subsection{LLM-based Behavior Tree generation}
Early language-planning approaches integrate LLMs with symbolic planners via
PDDL translation~\citep{liu2023llmp} or decompose instructions into sequential
steps~\citep{cao2023phasestep}. LLM-BRAin fine-tunes a transformer to generate
BTs from natural language using a predefined node
library~\citep{lykov2024llmbrain}; BTGenBot trains lightweight LLMs on
open-source BTs to produce executable XML~\citep{izzo2024btgenbot}.
LLM-as-BT-Planner and LLM-BT employ in-context learning and fine-tuning for
manipulation and assembly tasks~\citep{ao2025llmbt,zhou2024llmbt}. In social
robotics, integrated frameworks combine LLM-driven tree modification with BT
execution monitoring for runtime adaptation~\citep{merino2025behavior}. Across
all these approaches two limitations persist: successful implementations rely
on large proprietary models, and BT generation is treated as unconstrained text
synthesis, producing syntactically invalid trees that fail to compile on
edge-deployed SLMs~\citep{polimi_bt_thesis,zhang2023don}. No existing approach
addresses the KHTL knowledge-ingestion evaluation problem or provides a
contamination-free validation environment.

\subsection{Compositional generalisation}
The systematicity argument dates to~\cite{fodor1988connectionism};
SCAN~\citep{lake2018scan} operationalised it for sequence transduction. Modern
LLMs largely succeed on SCAN because its rules appear in natural-language
pre-training~\citep{drozdov2023compositional}. \order poses the harder
question: systematic generalisation over a physics \emph{provably} absent from
pre-training, under axiomatic constraints, with unique action-grounded outputs
that must ultimately compile and run on a manipulator.

\section{The \order Framework}
\label{sec:framework}

\begin{figure*}[t]
  \centering
  \begin{tikzpicture}[
    font=\footnotesize,
    corpusbox/.style={rectangle, rounded corners=4pt,
      draw=NavyBlue, fill=NavyBlue!9, text width=2.7cm,
      align=center, minimum height=0.9cm, inner sep=5pt},
    pipelinebox/.style={rectangle, rounded corners=4pt,
      draw=Purple, fill=Purple!9, text width=2.6cm,
      align=center, minimum height=0.9cm, inner sep=5pt},
    benchbox/.style={rectangle, rounded corners=4pt,
      draw=OliveGreen, fill=OliveGreen!9, text width=2.7cm,
      align=center, minimum height=0.9cm, inner sep=5pt},
    carrow/.style={-Stealth,thick,NavyBlue, shorten >=3pt, shorten <=3pt},
    parrow/.style={-Stealth,thick,Purple, shorten >=3pt, shorten <=3pt},
    barrow/.style={-Stealth,thick,OliveGreen,dashed, shorten >=3pt, shorten <=3pt},
    collabel/.style={font=\footnotesize\bfseries}
  ]
  \node[corpusbox, fill=BrickRed!12, draw=BrickRed] (ont) at (0,0)
    {\textbf{(0) Master Ontology}\\[1pt]
     \scriptsize contract $\cdot$ primitives\\ intrinsics $\cdot$ pairwise\\
     higher-order regimes};
  \node[corpusbox](ax) at(0,-1.7)
    {\textbf{(1) Axioms}\\[1pt]\scriptsize conservation laws\\dominance rules};
  \node[corpusbox](cf) at(0,-3.3)
    {\textbf{(2) Counterfactuals}\\[1pt]\scriptsize one violated constraint\\per entry};
  \node[corpusbox, fill=DarkGreen!10, draw=OliveGreen](sc) at(0,-5.0)
    {\textbf{(4) Spatial Configs}\\[1pt]\scriptsize 70 CPT scenes with\\analysis, $+$ $\pi^*$};
  \node[corpusbox](qp) at(0,-6.8)
    {\textbf{(3) QnA $\cdot$ (5) Paraphrases}\\[1pt]\scriptsize same knowledge,\\varied surface form};
  \draw[carrow](ont)--(ax);\draw[carrow](ax)--(cf);
  \draw[carrow](cf)--(sc);\draw[carrow](sc)--(qp);
  \node[collabel, NavyBlue, above=0.18cm of ont]{\orderworld corpus};
  \draw[decorate,decoration={brace,amplitude=5pt,mirror}, thick, BrickRed]
    (2.05,0.45)--(2.05,-7.15)
    node[midway,right=3pt,font=\scriptsize,BrickRed,align=left]
    {CPT input\\(\texttt{text} field)};
  \node[pipelinebox](base) at(6.6,0)
    {\textbf{Base SLM}\\[1pt]\scriptsize Mistral-7B\\LLaMA-3.1-8B\\Qwen3-4B};
  \node[pipelinebox](cpt) at(6.6,-2.2)
    {\textbf{CPT via LoRA}\\[1pt]\scriptsize knowledge adapter\\\emph{(K-LoRA)} on $\corpus$};
  \node[pipelinebox](merge) at(6.6,-4.4)
    {\textbf{Task Arithmetic Merge}\\[1pt]\scriptsize K-LoRA $+$ instruct\\$\Rightarrow$ KI-SLM};
  \node[pipelinebox, fill=Purple!16](sfine) at(6.6,-6.6)
    {\textbf{Skill SFT (S-LoRA)}\\[1pt]\scriptsize $\Rightarrow$ Domain-BT-LM};
  \draw[parrow](base)--(cpt);\draw[parrow](cpt)--(merge);
  \draw[parrow](merge)--(sfine);
  \node[collabel, Purple, above=0.18cm of base]{Adaptation pipeline};
  \node[benchbox](b1) at(12.6,0)
    {\textbf{MCQ: 7 Categories}\\[1pt]
     \scriptsize Intr.$\cdot$QnA$\cdot$Ax.$\cdot$Pair.\\
     CF$\cdot$H.O.$\cdot$\textbf{\textcolor{BrickRed}{Spatial}}};
  \node[benchbox](b2) at(12.6,-2.2)
    {\textbf{Bloom Weights}\\[1pt]\scriptsize $1.0\!\times$--$3.0\!\times$\\
     Agg.\ Score (Eq.\,\ref{eq:agg})};
  \node[benchbox, fill=BrickRed!10, draw=BrickRed](split) at(12.6,-4.4)
    {\textbf{\orderspatial}\\[1pt]\scriptsize 70 seen / 183 unseen\\$\HA$ vs $\HB$ axis};
  \node[benchbox, fill=NavyBlue!12, draw=NavyBlue](metrics) at(12.6,-6.6)
    {\textbf{Metrics}\\[1pt]\scriptsize MCQ Agg.$\cdot$PPL$\cdot$PAR\\
     \tauorder$\cdot$\svr$\cdot$\sdhr};
  \draw[barrow](b1)--(b2);\draw[barrow](b2)--(split);
  \draw[barrow](split)--(metrics);
  \node[collabel,OliveGreen,above=0.18cm of b1]{\orderbench};
  \draw[-Stealth,thick,BrickRed,dashed,shorten >=2pt]
    (sfine.east)--++(0.5,0)|-(metrics.west);
  \end{tikzpicture}
  \caption{%
    The \order framework.
    \textbf{Left:} the \orderworld knowledge pyramid; all upstream context is
    injected into downstream generation calls.
    \textbf{Centre:} the two-stage adaptation pipeline, instantiated across
    seven \klora variants for three SLM families and then extended with
    \slora skill fine-tuning to produce Domain-BT-LMs
    (Sections~\ref{sec:arch}--\ref{sec:slora}).
    \textbf{Right:} \orderbench, comprising the MCQ component and the
    \orderspatial seen/unseen ordering task, with the metric suite of
    Section~\ref{sec:metrics}.}
  \label{fig:framework}
\end{figure*}

\subsection{Visual grammar and primitive design}
\label{sec:grammar}

\orderworld represents its world through a \emph{visual grammar} of primitives:
\begin{equation}
\small
\mathcal{G}
=
\underbrace{\mathcal{C}}_{\text{9 colours}}
\times
\underbrace{\mathcal{S}}_{\text{5 shapes}}
\times
\underbrace{\mathcal{Z}}_{\text{3 sizes}}
\times
\underbrace{\mathcal{T}}_{\text{3 textures}}
\times
\underbrace{\mathcal{R}}_{\text{5 relations}},
\label{eq:grammar}
\end{equation}
yielding
\[
|\mathcal{G}| = 9 \times 5 \times 3 \times 3 \times 5 = 2{,}025
\]
distinct primitive configurations. The corresponding single-object descriptor
space, excluding relations, contains
$9 \times 5 \times 3 \times 3 = 405$ distinct object signatures, while the
five relation types govern pairwise interactions between objects. Separately,
when composing multi-object scenes from the 25 primitive tokens
($9 + 5 + 3 + 3 + 5 = 25$), there are
$2^{25} - 1 = 33{,}554{,}431$ possible non-empty subsets of primitives.
Thus, the $2^{25} - 1$ quantity characterizes the \emph{scene-composition
space}, and is distinct from the single-primitive space
$|\mathcal{G}| = 2{,}025$ defined in Eq.~\ref{eq:grammar}. Primitive tokens
are visually grounded and identifiable by off-the-shelf vision-language
models (VLMs), while all interaction semantics are fictional and absent
from any pre-training corpus, enabling a clean perception--cognition
separation that Section~\ref{sec:perception} exploits.

\subsection{\orderworld corpus architecture}
\label{sec:corpus}

\orderworld is a six-layer knowledge pyramid
(Figure~\ref{fig:framework}, Table~\ref{tab:corpus}), constructed so that every
downstream generation call receives all upstream layers as injected context via
the GPT-4.1 oracle, ensuring global closed-world consistency throughout.
\textbf{Layer~0} (Master Ontology) forms the physics foundation: primitive
vocabulary, intrinsic semantics, pairwise interaction laws, and higher-order
emergent regimes. \textbf{Layers~1--2} (Axioms, Counterfactuals) formalise
global constraints as logic rules and stress-test them via single-constraint
violations. \textbf{Layers~3 and~5} (QnA, Paraphrases) re-express the same
physics in varied surface formats following the multi-format CPT principle
of~\cite{allen2023physlm}. The epistemically richest layer is
\textbf{Layer~4} (Spatial Configurations): 70 CPT scenes, each presenting
$n\!\in\![2,5]$ objects with full inter-/intra-object interaction analysis,
ontology and axiom citations, and a gold priority order~$\pi^{*}$. The extended
253-scene pool covers CPT training, skill fine-tuning
(Section~\ref{sec:slora}), and \orderspatial evaluation. Full generation
specifications, including all system prompts, appear in
Appendix~\ref{app:generation}.

\begin{table}[t]
\centering
\caption{\orderworld token distribution ($\corpus = 342{,}069$ tokens).}
\label{tab:corpus}
\setlength{\tabcolsep}{4pt}
\renewcommand{\arraystretch}{1.05}
\footnotesize
\begin{tabular}{@{}lrl@{}}
\toprule
\textbf{Section} & \textbf{Tokens} & \textbf{Role} \\
\midrule
Master Ontology (0) & 67,389  & Physics foundation \\
Axioms (1)          & 11,886  & Global constraints \\
Counterfactuals (2) & 16,336  & Constraint stress-tests \\
QnA (3)             & 47,836  & Format diversification \\
Spatial Configs (4) & 100,214 & Scenes \& priority orders \\
Paraphrases (5)     & 98,408  & Format diversification \\
\midrule
\textbf{Total} & \textbf{342,069} & CPT corpus $\corpus$ \\
\bottomrule
\end{tabular}
\end{table}

\section{\orderbench, \orderspatial, and Evaluation Metrics}
\label{sec:benchmark}

\subsection{MCQ component: knowledge extraction fidelity}
\label{sec:mcq}

\orderbench comprises 500 four-option multiple-choice questions (MCQ)
measuring \emph{knowledge extraction fidelity}: how completely and accurately a
model can recover and apply the domain physics from $\corpus$. Every question
has a gold answer directly recoverable from corpus text; compositional
generalisation to novel configurations is evaluated separately by
\orderspatial. Questions span seven categories derived from every corpus layer
(Table~\ref{tab:categories}); per-category distractor design is specified in
Appendix~\ref{app:gen:mcq}.

\begin{table}[t]
\centering
\caption{\orderbench MCQ composition and Bloom weights. Total weight mass:
1127.5; random baseline: 25\%.}
\label{tab:categories}
\setlength{\tabcolsep}{5pt}
\renewcommand{\arraystretch}{1.05}
\footnotesize
\begin{tabular}{@{}lccc@{}}
\toprule
\textbf{Category} & $N$ & $w_c$ & $w_cN$ \\
\midrule
Intrinsic Semantics & 10  & 1.0 & 10.0 \\
QnA                 & 30  & 1.0 & 30.0 \\
Axioms              & 60  & 1.5 & 90.0 \\
Spatial (MCQ)       & 72  & 2.0 & 144.0 \\
Pairwise            & 150 & 2.5 & 375.0 \\
Counterfactuals     & 111 & 2.5 & 277.5 \\
Higher-Order        & 67  & 3.0 & 201.0 \\
\midrule
\textbf{Total} & \textbf{500} & & \textbf{1127.5} \\
\bottomrule
\end{tabular}
\end{table}

\subsection{\orderspatial: seen/unseen priority ordering}
\label{sec:orderspatial}

\orderspatial is the more discriminative evaluation component. The model
receives a full scene description in the visual grammar and must output the
complete safe manipulation sequence $\hat{\pi}$ for all $n$ objects. The gold
sequence $\pi^*$ is unique by Proposition~\ref{prop:total_order}
(Appendix~\ref{app:order_theory}): objects are ranked by a domain-induced
criticality score combining intrinsic and pairwise contributions, with a
deterministic lexicographic tiebreaker (Colour $\succ$ Shape $\succ$ Size
$\succ$ Texture) ensuring totality.

\textbf{Scene pool.} The 253-scene pool is partitioned into 70 \emph{seen}
scenes (full text appeared in $\corpus$) and 183 \emph{unseen} scenes never
part of any training data. Welch $t$-tests across all structural and physics
complexity axes confirm neither split is significantly harder (all $|t|<2$,
$p>0.05$; Table~\ref{tab:complexity_distribution},
Appendix~\ref{app:complexity}); the mild four-object skew in unseen ($30.6\%$
vs.\ $18.6\%$ seen) slightly disfavours unseen performance, strengthening the
$\HB$ interpretation.

\textbf{The $\HA$/$\HB$ axis.} Under $\HA$ (memorisation), seen performance
improves but unseen stagnates
($\Delta\tauorder^{\mathrm{seen}} \gg \Delta\tauorder^{\mathrm{unseen}}$).
Under $\HB$ (world-model induction), \emph{both} seen and unseen improve
substantially and comparably
($\Delta\tauorder^{\mathrm{seen}} \approx \Delta\tauorder^{\mathrm{unseen}}$).

\subsection{Evaluation metrics}
\label{sec:metrics}

All metrics used in this paper are defined here so that every subsequent table
is readable without forward reference.

\textbf{Bloom-weighted MCQ aggregate ($S_{\mathrm{agg}}$).}
\begin{equation}
  S_{\mathrm{agg}} = \frac{\sum_c w_c \cdot \mathrm{Acc}_c}{\sum_c w_c},
  \label{eq:agg}
\end{equation}
with per-category weights $w_c$ from Table~\ref{tab:categories} reflecting
reasoning depth; random guessing yields $S_{\mathrm{agg}} = 25\%$.

\textbf{Perplexity (PPL).} Token-level perplexity on a held-out split of
$\corpus$, used only as a training-monitoring signal; Section~\ref{sec:ppl_diss}
shows it is not a reliable proxy for physics internalisation.

\textbf{Normalised Kendall concordance (\tauorder).} The primary metric
throughout. It measures the fraction of object pairs in a predicted order
$\hat{\pi}$ that are correctly ranked relative to $\pi^*$:
\begin{equation}
  \tauorder \;=\; \frac{C}{C+D} \;=\; \frac{\tau_b+1}{2} \;\in\; [0,1],
  \label{eq:tau}
\end{equation}
where $C$ and $D$ are the concordant and discordant pair counts and $\tau_b$ is
Kendall's Tau. $\tauorder = 0.500$ is random and $1.000$ is perfect. On
\orderspatial we write \tauorders and \tauorderu for the seen and unseen
splits. In BT evaluation the same statistic is computed over the manipulation
order encoded by the generated tree, so a single metric traces the pipeline
from free-form ordering to executable plan. The empirical random baselines on
our 253-scene pool confirm the theoretical value to three decimal places
($\tauorder = 0.500\pm0.004$ seen, $0.500\pm0.003$ unseen;
Appendix~\ref{app:baseline}), validating the scene-complexity balance of the
partition.

\textbf{Positional accuracy (PAR, P-PAR).} $\mathrm{PAR}_{\mathrm{mean}}$
(written \parcmd) measures exact positional accuracy;
$\mathrm{PAR}_{\mathrm{perfect}}$ (P-PAR) measures the fraction of scenes with
entirely correct sequences, reported downstream as $n$/82. PAR floors are
split-specific ($0.333$ seen, $0.309$ unseen; Appendix~\ref{app:baseline}).

\textbf{Schema Validity Rate (\svr).} The fraction of generated BTs that pass
both XML parsing and the \texttt{BehaviorTree.CPP}~v4 schema check (correct
root/\texttt{BehaviorTree} structure, known node identifiers, valid attribute
enumerations):
\begin{equation}
  \svr = \frac{\bigl|\{i : \mathrm{valid\_L1}(\mathrm{BT}_i)\}\bigr|}
              {\bigl|\{i : \mathrm{attempted}(\mathrm{BT}_i)\}\bigr|}.
\end{equation}
\svr is a prerequisite gate: a tree failing the schema check is
non-executable, its priority order is undefined, and it contributes
$\tauorder = 0$ to all aggregates.

\textbf{Stack Dependency Handling Rate (\sdhr).} Whether the model inserts the
required accessibility \texttt{Fallback} guard for every scene containing an
\texttt{on\_top\_of} or \texttt{stacked\_with} constraint:
\begin{equation}
\scriptsize
  \sdhr = \frac{\bigl|\{i : \mathrm{has\_stack}(\mathrm{scene}_i)
                  \wedge \mathrm{has\_ObjectAccessible}(\mathrm{BT}_i)\}\bigr|}
         {\bigl|\{i : \mathrm{has\_stack}(\mathrm{scene}_i)\}\bigr|}.
\end{equation}
\texttt{has\_ObjectAccessible} is necessary but not sufficient: it confirms
that a stacking guard was inserted, not that it targeted the correct blocked
object. Both \svr and \sdhr target $1.0$ on the gold dataset; deviations
indicate structural failures that precede any assessment of priority-ordering
correctness.

\textbf{A deliberate asymmetry in strictness.} In \orderspatial, model outputs
containing a priority tuple anywhere in free-form text are re-parsed by a
GPT-4o-mini extraction call, so that all 253 scenes contribute to \tauorder
regardless of surface formatting variation; the objective there is to test
compositional generalisation and physics internalisation, not the ability to
produce structured outputs. In BT generation no analogous recovery is possible,
because an ill-formed tree is physically non-executable; the schema boundary is
therefore hard. These different objectives justify the different strictness of
evaluation across the two stages.

\section{Knowledge Ingestion Variants}
\label{sec:ki_variants}

Seven \klora variants are evaluated across three SLM families
(Table~\ref{tab:ki_desc}). All use LoRA~\citep{hu2022lora} on all linear layers
($r\!=\!64$, $\alpha\!=\!128$, dropout $0$, lr $=5\!\times\!10^{-5}$, 10
epochs, effective batch size 4) and are merged onto the instruct model via task
arithmetic~\citep{ilharco2023editing} ($\alpha_1\!=\!\alpha_2\!=\!0.5$).
Full algorithmic specifications appear in Appendix~\ref{app:algorithms}.
\emph{Combined} denotes the task-arithmetic merge of the independently trained
SSR-CLM and RAFT adapters; in the downstream evaluation we additionally report
the reverse merge ordering, writing 12-Combined for the ordering used in the
benchmark evaluation and 21-Combined for its reverse, giving eight variants in
total at the skill-adaptation stage.

\begin{table}[t]
\centering
\caption{\klora knowledge-ingestion variants.}
\label{tab:ki_desc}
\footnotesize
\setlength{\tabcolsep}{3pt}
\renewcommand{\arraystretch}{1.02}
\begin{tabular}{@{}p{0.16\columnwidth}p{0.78\columnwidth}@{}}
\toprule
\textbf{Variant} & \textbf{Description} \\
\midrule
Normal   & Packed causal LM over all $\corpus$ sections; all-token CLM loss. \\
Overfit  & Normal extended to near-zero training loss ($\mathcal{L}\!\approx\!0$). \\
\dkl\textsuperscript{*}     
         & Base$\,{\leftarrow}\,$instruct embedding replacement before
           CPT~\citep{bhushan2026dkl}. \\
RAFT     & Retrieval-augmented fine-tuning with hard-negative
           distractors~\citep{zhang2024raft}. \\
SSR-CLM  & Packed CLM on non-spatial entries; loss-masked SSR (Structured Scene
           Reasoning) on spatial entries (loss on assistant tokens only,
           preventing format overfitting). \\
Unified  & SSR-CLM $+$ RAFT $+$ DPO pairs (chosen\,=\,oracle-faithful,
           rejected\,=\,parametric-bias) in one training script. \\
Combined & Task-arithmetic merge of two independently trained adapters
           (SSR-CLM \& RAFT):
           $\theta_{KI} = \theta_I + 0.5\,\tau_{\mathrm{SSR}}
           + 0.5\,\tau_{\mathrm{RAFT}}$, where $\theta_I$ is the off-the-shelf
           instruct SLM. \\
\bottomrule
\end{tabular}

\vspace{2pt}
\parbox{0.94\columnwidth}{\footnotesize
\textsuperscript{*}\textbf{Decoupled Knowledge Learning (DKL).}
}
\end{table}

\section{Benchmark Experiments}
\label{sec:experiments}

\subsection{Models and evaluation protocol}

We evaluate Mistral-7B-Instruct-v0.3~\citep{jiang2023mistral7b} (M-7B),
LLaMA-3.1-8B-Instruct~\citep{grattafiori2024llama} (L-8B), and
Qwen3-4B-Instruct-2507~\citep{qwen2023} (Q-4B), a hybrid reasoning
architecture with thinking mode enabled by default.\footnote{Thinking mode in
Qwen3 reflects realistic deployment for hybrid-reasoning SLMs; see
Appendix~\ref{app:qwen_thinking} for the evaluation-fairness discussion.}
GPT-4.1 serves as the frontier reference in both closed-book and retrieval
settings. Full PPL profiles and per-section MCQ breakdowns for all \klora
variants are in Appendices~\ref{app:ppl} and~\ref{app:mcq_detailed}.

\subsection{Closed-book baselines and retrieval validation}
\label{sec:zeroshot}

\begin{table}[t]
  \centering
  \caption{Closed-book zero-shot performance (PPL on a held-out split of
    $\corpus$). \textbf{Sp.MCQ}$\,{\gg}\,$$S_{\mathrm{agg}}$ for all SLMs
    reflects common-sense ordering priors, not domain knowledge; the gap closes
    only after CPT (Appendix~\ref{app:mcq_detailed}).}
  \label{tab:baseline_closed}
  \setlength{\tabcolsep}{6pt}
  \renewcommand{\arraystretch}{1.05}
  \footnotesize
  \begin{tabular}{lccc}
    \toprule
    \textbf{Model} & \textbf{PPL}$\downarrow$ & $S_{\mathrm{agg}}$
      & \textbf{Sp.MCQ} \\
    \midrule
    GPT-4.1 & NA    & 70.4 & 70.8 \\
    M-7B    & 13.72 & 48.9 & 65.3 \\
    L-8B    & 21.79 & 34.7 & 44.4 \\
    Q-4B    & 24.22 & 57.0 & 69.4 \\
    \bottomrule
  \end{tabular}
\end{table}

\begin{table}[t]
  \centering
  \caption{\orderbench MCQ with retrieval (global hybrid BM25\,$+$\,dense
    pipeline, Appendix~\ref{app:rag_mcq}). Non-spatial categories saturate
    ($\geq\!90\%$), validating corpus coherence. Spatial remains harder for
    smaller models.}
  \label{tab:baseline_rag}
  \setlength{\tabcolsep}{5pt}
  \renewcommand{\arraystretch}{1.05}
  \footnotesize
  \begin{tabular}{lcccc}
    \toprule
    \textbf{Category} & GPT-4.1 & M-7B & L-8B & Q-4B \\
    \midrule
    Axioms       & 100.0 & 96.67 & 98.33 & 98.33 \\
    Intrinsic    & 100.0 & 100.0 & 100.0 & 100.0 \\
    Pairwise     & 100.0 & 98.67 & 96.67 & 98.67 \\
    Higher-Order & 100.0 & 100.0 & 97.01 & 100.0 \\
    Counterfact. &  98.2 & 93.69 & 90.09 &  96.4 \\
    QnA          & 100.0 & 100.0 & 100.0 & 100.0 \\
    \textbf{Spatial}
                 & \textbf{86.11} & \textbf{68.06}
                 & \textbf{51.39} & \textbf{72.22} \\
    \midrule
    $S_{\mathrm{agg}}$ & 97.78 & 93.66 & 89.58 & 94.99 \\
    \bottomrule
  \end{tabular}
\end{table}

Table~\ref{tab:baseline_closed} shows closed-book baselines. LLaMA-3.1-8B
achieves $S_{\mathrm{agg}} = 34.7\%$, barely above the 25\% random baseline,
confirming that \orderworld's fictional physics is genuinely absent from its
pre-training. Spatial MCQ consistently outperforms factual recall categories at
zero-shot for all SLMs (e.g.\ Mistral: $65.3\%$ spatial vs.\ $25.0\%$ Axioms),
because spatial ordering questions resemble common-sense tasks in pre-training
data.

Table~\ref{tab:baseline_rag} reports performance with retrieval. GPT-4.1
reaches $S_{\mathrm{agg}} = 97.78\%$, establishing that \orderbench is
essentially fully answerable given relevant context and confirming corpus
coherence. All non-spatial categories saturate at $\geq\!90\%$ across models.
The spatial MCQ category remains harder under retrieval (GPT-4.1: $86.11\%$;
SLMs: $51$--$72\%$), reflecting that compositional spatial questions require
integrating multiple retrieved rules simultaneously, a task more demanding for
smaller models under identical retrieval. This gap across models under
identical retrieval conditions indicates the bottleneck is reasoning capacity,
not retrieval quality (pipeline details in Appendix~\ref{app:rag_mcq}).

\subsection{The PPL--spatial dissociation}
\label{sec:ppl_diss}

SSR-CLM achieves the lowest PPL across all families (Mistral: 1.52, LLaMA:
2.52, Qwen: 4.62) yet produces the weakest spatial MCQ among non-collapsed
Mistral variants ($43.1\%$), establishing that perplexity is not a reliable
proxy for physics internalisation. Mistral-Unified collapses catastrophically
($S_{\mathrm{agg}} = 1.8\%$) while retaining PPL of 2.15, with no warning from
training curves, underscoring that post-merge evaluation on \orderbench is a
necessary validation step (full profiles in Appendix~\ref{app:ppl}).

\subsection{\klora MCQ summary}

Best MCQ aggregate performances after CPT (SSR-CLM: L-8B $78.6\%$, Q-4B
$87.5\%$; \dkl: Q-4B $86.3\%$) exceed the zero-shot GPT-4.1 baseline
($70.4\%$), demonstrating that targeted CPT on a structured synthetic corpus
can equip SLMs with domain knowledge-extraction capacity approaching that of a
frontier model. Full per-category \klora MCQ results appear in
Appendix~\ref{app:mcq_detailed}.

\subsection{\orderspatial results}
\label{sec:orderspatial_results}

Table~\ref{tab:spatial_all} reports \orderspatial for every variant and family,
and Table~\ref{tab:spatial_improvement} summarises the improvement of the best
variant per family.

\begin{table*}[t]
  \centering
  \caption{\orderspatial: \tauorder and \parcmd for all \klora variants across
    the three SLM families, on the 70 seen and 183 unseen scenes. Best per
    family in \textbf{bold}. The random baseline is $\tauorder = 0.500$ on both
    splits (Appendix~\ref{app:baseline}); \parcmd floors are $0.333$ seen and
    $0.309$ unseen. Off-the-shelf Qwen3 already exceeds GPT-4.1; Mistral-RAFT
    and LLaMA-Unified show inverted seen/unseen gaps.}
  \label{tab:spatial_all}
  \setlength{\tabcolsep}{4.5pt}
  \renewcommand{\arraystretch}{1.08}
  \footnotesize
  \begin{tabular}{l cccc cccc cccc}
    \toprule
    & \multicolumn{4}{c}{\textbf{Mistral-7B (M-7B)}}
    & \multicolumn{4}{c}{\textbf{LLaMA-3.1-8B (L-8B)}}
    & \multicolumn{4}{c}{\textbf{Qwen3-4B (Q-4B)}} \\
    \cmidrule(lr){2-5}\cmidrule(lr){6-9}\cmidrule(lr){10-13}
    & \multicolumn{2}{c}{Seen (70)} & \multicolumn{2}{c}{Unseen (183)}
    & \multicolumn{2}{c}{Seen (70)} & \multicolumn{2}{c}{Unseen (183)}
    & \multicolumn{2}{c}{Seen (70)} & \multicolumn{2}{c}{Unseen (183)} \\
    \cmidrule(lr){2-3}\cmidrule(lr){4-5}\cmidrule(lr){6-7}\cmidrule(lr){8-9}
    \cmidrule(lr){10-11}\cmidrule(lr){12-13}
    \textbf{Variant}
      & \tauorder & \parcmd & \tauorder & \parcmd
      & \tauorder & \parcmd & \tauorder & \parcmd
      & \tauorder & \parcmd & \tauorder & \parcmd \\
    \midrule
    GPT-4.1 (reference)
      & \multicolumn{12}{c}{\emph{family-independent:}\quad
        $\tauorders = 0.441$, $\parcmd^{s} = 0.301$; \quad
        $\tauorderu = 0.415$, $\parcmd^{u} = 0.248$} \\
    \midrule
    Instruct
      & 0.528 & 0.368 & 0.509 & 0.347
      & 0.513 & 0.366 & 0.554 & 0.345
      & 0.582 & 0.394 & 0.597 & 0.412 \\
    \midrule
    Normal
      & 0.49 & 0.38 & 0.55 & 0.36
      & 0.51 & 0.36 & 0.54 & 0.35
      & 0.64 & 0.44 & 0.65 & 0.44 \\
    Overfit
      & 0.43 & 0.36 & 0.55 & 0.35
      & 0.54 & 0.37 & 0.54 & 0.35
      & 0.63 & 0.45 & 0.64 & 0.45 \\
    \dkl
      & 0.67 & 0.50 & 0.69 & 0.49
      & 0.56 & 0.41 & 0.58 & 0.39
      & 0.74 & 0.57 & 0.72 & 0.54 \\
    RAFT
      & \textbf{0.71} & \textbf{0.55} & \textbf{0.73} & \textbf{0.50}
      & 0.56 & 0.42 & 0.60 & 0.36
      & 0.66 & 0.47 & 0.68 & 0.48 \\
    SSR-CLM
      & 0.59 & 0.46 & 0.51 & 0.34
      & 0.76 & \textbf{0.66} & 0.74 & 0.59
      & 0.76 & 0.61 & 0.70 & 0.49 \\
    Unified
      & 0.49 & 0.34 & 0.46 & 0.32
      & \textbf{0.77} & 0.65 & \textbf{0.81} & \textbf{0.71}
      & 0.63 & 0.46 & 0.63 & 0.43 \\
    Combined
      & 0.51 & 0.34 & 0.52 & 0.34
      & 0.76 & 0.60 & 0.70 & 0.48
      & \textbf{0.79} & \textbf{0.69} & \textbf{0.74} & \textbf{0.55} \\
    \bottomrule
  \end{tabular}
\end{table*}

\begin{table}[t]
  \centering
  \caption{Absolute \tauorder improvement to the best \klora variant per
    family. Both seen and unseen improve substantially, providing evidence
    against memorisation ($\HA$).}
  \label{tab:spatial_improvement}
  \setlength{\tabcolsep}{3.5pt}
  \renewcommand{\arraystretch}{1.1}
  \footnotesize
  \begin{tabular}{l l cc cc cc}
    \toprule
    \textbf{Fam.} & \textbf{KI}
      & \multicolumn{2}{c}{\textbf{Instruct}}
      & \multicolumn{2}{c}{\textbf{Post-KI}}
      & \multicolumn{2}{c}{\textbf{Gain}} \\
    \cmidrule(lr){3-4}\cmidrule(lr){5-6}\cmidrule(lr){7-8}
    & & \tauorders & \tauorderu & \tauorders & \tauorderu
      & $\Delta^s$ & $\Delta^u$ \\
    \midrule
    M-7B & RAFT     & 0.528 & 0.509 & 0.71 & 0.73 & $+$0.18 & $+$0.21 \\
    L-8B & Unified  & 0.513 & 0.554 & 0.77 & 0.81 & $+$0.26 & $+$0.26 \\
    Q-4B & Combined & 0.582 & 0.597 & 0.79 & 0.74 & $+$0.21 & $+$0.14 \\
    \bottomrule
  \end{tabular}
\end{table}

\paragraph{Finding 1: pre-training priors conflict with \orderworld physics}
GPT-4.1 without domain adaptation scores $\tauorder = 0.441$ (seen) and $0.415$
(unseen) on \orderspatial; both below the $0.500$ random baseline. The
sub-random unseen result is statistically significant ($z = -2.30$,
$p = 0.011$; Appendix~\ref{app:complexity}), providing strong evidence for
contamination-free evaluation: the frontier model's pre-training priors
actively and systematically conflict with the fictitious physics.

\paragraph{Finding 2: CPT improves \emph{both} seen and unseen \tauorder}
Under pure memorisation ($\HA$), CPT would improve seen performance but leave
unseen stagnant. The observed pattern is categorically different:
Mistral-RAFT achieves $\Delta\tauorder^s = +0.18$, $\Delta\tauorder^u = +0.21$
(unseen improves \emph{more} than seen); LLaMA-Unified achieves
$\Delta\tauorder^s = +0.26$, $\Delta\tauorder^u = +0.26$ (near-equal,
symmetric internalisation); Qwen-Combined achieves $\Delta\tauorder^s = +0.21$,
$\Delta\tauorder^u = +0.14$ (both improve substantially from an already-high
baseline).

\paragraph{Finding 3: post-KI seen/unseen gaps are small, and two of three
best variants show an inverted gap}
Mistral-RAFT: $0.71 - 0.73 = -0.02$ (inverted). LLaMA-Unified:
$0.77 - 0.81 = -0.04$ (inverted). Qwen-Combined: $0.79 - 0.74 = +0.05$.
LLaMA-Unified's equal absolute improvements ($+0.26$ on both splits) provide
the most robust evidence for $\HB$.

\paragraph{Qwen deployment recommendation}
Although Combined achieves the strongest \orderspatial scores for Qwen,
Appendix~\ref{app:general_benchmarks} reveals that two sequential task-arithmetic
merges produce the largest general-capability degradation across all four
general benchmarks ($-10.2$, $-3.1$, $-4.9$, $-11.8$ points on MMLU-Pro, GSM8K,
BBH, IFEval). \dkl-Qwen avoids this: highest MCQ aggregate ($86.3\%$), strong
downstream BT performance ($\tauorder\!=\!0.813$, Section~\ref{sec:slora}), and
substantially better retention.

\section{From Knowledge to Action: System Architecture}
\label{sec:arch}

The remainder of the paper carries the validated models from benchmark scores
to executable plans on a manipulator. The system decouples perception,
cognition, and execution across three layers. This decoupling is deliberate: it
isolates visual understanding from symbolic planning, allows independent
evaluation of each layer, and makes domain knowledge swappable without
re-training the perception or execution modules.



\subsection{Scene perception via visual grammar}
\label{sec:perception}

A lightweight off-the-shelf VLM, SigLIP2~\citep{tschannen2025siglip},
serves as the scene analyser. Rather than fine-tuning perception for a
specific domain, the visual grammar $\mathcal{G}$ of
Eq.~\ref{eq:grammar} provides a fixed set of 25 primitives that are
directly grounded in the scene. The VLM maps each image to a structured
JSON representation of objects, their attributes, spatial/physical
constraints, and inter-object relations. Domain knowledge---including
interaction physics, safety orderings, and hazard hierarchies---is encoded
entirely in terms of these primitives through \orderworld, so the
cognitive module never requires raw image access.

This separation makes perception domain-agnostic and allows the benchmark
to evaluate knowledge and reasoning independently of perceptual
limitations, a bottleneck noted in ELLMER and BTGenBot
\citep{mon2025embodied,polimi_bt_thesis}. The VLM prompt is provided in Appendix~\ref{app:vlm-prompt}.

\subsection{Cognitive layer: two-stage adaptation}
\label{sec:cognitive}

In Stage~1 (\klora knowledge ingestion, Section~\ref{sec:ki_variants}), a LoRA
adapter trained on \orderworld via continual pre-training on the base model is
merged onto the instruct model via task arithmetic~\citep{ilharco2023editing}:
\begin{equation}
  \theta_{\text{KI}} \;=\; \theta_{\text{inst}} + \lambda\,\tau_{\klora},
  \quad \lambda = 0.5 .
  \label{eq:ki_merge}
\end{equation}
This yields a \emph{knowledge-ingested} (KI) model with strong \orderbench and
\orderspatial performance without degrading general instruction-following
(Appendix~\ref{app:general_benchmarks}). In Stage~2 (\slora skill fine-tuning,
Section~\ref{sec:slora}), a second LoRA adapter is trained for the skill of BT
generation and merged onto the KI model to produce the final
\emph{Domain-BT-LM}:
\begin{equation}
  \theta_{\text{Domain-BT-LM}}
  \;=\; \theta_{\text{inst}}
  + \underbrace{\Delta\theta_{\klora}}_{\text{knowledge}}
  + \underbrace{\Delta\theta_{\slora}}_{\text{skill}}.
  \label{eq:domainbt}
\end{equation}
The two stages, knowledge and skill, are fully
disentangled~\citep{liu2024fictitious}: they are trained on separate objectives
and composed via task arithmetic. KI models have passed through Stage~1 only;
Domain-BT-LMs have passed through both.

\subsection{Execution layer: Behavior Tree executor}
\label{sec:executor}

Generated BTs are serialised as XML conforming to the
\texttt{BehaviorTree.CPP}~(v4) schema. The executor parses the tree, exposes
action nodes as ROS service clients, and propagates ticks at a configurable
frequency. Action nodes interface directly with manipulator motion planners
(MoveIt) and a GraspNet-based grasp planner. Each leaf action returns
\texttt{Success}, \texttt{Failure}, or \texttt{Running}; control-flow nodes
(\texttt{Sequence}, \texttt{Fallback}, \texttt{Parallel}) aggregate these
statuses upward. The executor is deployed on the simulated \texttt{iiwa7} arm
in Gazebo (Section~\ref{sec:hitl}) and is hardware-ready for the Doosan A0509S.

\section{Knowledge--Skill Separation and Skill Fine-Tuning}
\label{sec:slora}

\subsection{KI-models without fine-tuning: motivating \slora}
\label{sec:motivating}

A natural starting hypothesis is that a model with strong \orderspatial
performance might zero-shot generate executable BTs encoding those orderings.
We test this directly before any \slora training, evaluating all KI variants on
82 novel test scenes under the same XML prompt used for \slora evaluation.

Table~\ref{tab:zeroshot} reports results. Across Mistral-7B and LLaMA-3.1-8B,
\svr is essentially zero for nearly all KI-adapted variants. The LLaMA instruct
baseline itself achieves only 12.2\% \svr, underscoring that XML schema
compliance is a distinct learned skill absent from standard instruction tuning.
One partial exception is LLaMA-Normal ($\svr = 70.7\%$, $\tauorder = 0.524$):
ordinary packed-CLM pre-training modestly preserves instruction-following
format, but without structured reasoning objectives it does not translate into
compositionally correct BTs. Qwen3-4B is the notable exception: the vanilla
instruct achieves $\svr = 97.6\%$ and $\tauorder = 0.651$, confirming that
Qwen's built-in chain-of-thought activates structural reasoning, alongside its
strong instruction-following capability (the highest off-the-shelf IFEval
performance among the three SLMs here). However, even Qwen's zero-shot
\tauorder is far below the \slora-trained Domain-BT-LMs of
Section~\ref{sec:slora_results}, and the KI variants degrade gracefully but
meaningfully - \dkl drops to $\svr = 61.0\%$ and $\tauorder = 0.012$,
suggesting that the embedding-replacement procedure that was so effective for
MCQ recall actively disrupts Qwen's zero-shot XML generation mode.

\begin{table}[t]
  \centering
  \caption{\textbf{Zero-shot BT generation of KI-models} (82 novel samples, no
    fine-tuning). Near-zero \svr for Mistral and LLaMA variants confirms that
    domain knowledge and BT-generation skill are separable capabilities.
    GPT-4.1 vanilla is included as a frontier reference.}
  \label{tab:zeroshot}
  \setlength{\tabcolsep}{5pt}
  \renewcommand{\arraystretch}{1.05}
  \footnotesize
  \begin{tabular}{llccc}
    \toprule
    \textbf{Family} & \textbf{Variant} & \svr & \tauorder & \sdhr \\
    \midrule
    \multirow{4}{*}{L-8B}
      & Instruct       & 12.2 & 0.057 & 53.4 \\
      & Normal         & 70.7 & 0.524 & 72.4 \\
      & RAFT           & 31.7 & 0.216 & 87.9 \\
      & Unified        &  4.9 & 0.037 & 53.4 \\
    \midrule
    \multirow{4}{*}{M-7B}
      & Instruct       &  2.4 & 0.019 & 25.9 \\
      & Normal         &  4.9 & 0.035 & 44.8 \\
      & RAFT           &  0.0 & 0.000 &  0.0 \\
      & Unified        &  1.2 & 0.012 & 19.0 \\
    \midrule
    \multirow{3}{*}{Q-4B}
      & Instruct       & 97.6 & 0.651 & 100.0 \\
      & Normal         & 86.6 & 0.480 & 62.1  \\
      & \dkl           & 61.0 & 0.012 & 15.5  \\
    \midrule
    \multicolumn{2}{l}{GPT-4.1 (Vanilla)} & 96.3 & 0.480 & 100.0 \\
    \bottomrule
  \end{tabular}
\end{table}

\textbf{Why knowledge-ingested models cannot generate BTs.}
The failure pattern is consistent and informative. Models output free-form
analysis prose - often the kind of step-by-step interaction reasoning that
the CPT training \emph{specifically trained them to produce} via the SSR-CLM
and Unified objectives. The very objective that made these models good at
spatial reasoning (producing structured scene analyses in the Layer~4 format)
actively competes with the XML output format required for BT generation. This
\emph{knowledge--skill separation} is a concrete manifestation of the abstract
principle that domain knowledge and task-format skill are distinct capabilities
requiring distinct adaptation objectives (Eq.~\ref{eq:domainbt}). Injecting
retrieved ontology chunks on top of zero-shot KI-models without \slora
similarly failed: the retrieved text further encouraged analytical prose,
driving \svr lower rather than higher - a failure mode we return to in
Section~\ref{sec:hijacking}. This finding motivates skill
fine-tuning as a necessary second stage.

\subsection{Training data: NL-query and gold-BT dataset}

The \slora training set is derived from the \order scene pool. The \orderworld
corpus includes 212 training scenes: 70 CPT scenes (Layer~4 of \orderworld) and
an additional 142 scenes generated specifically for \slora training, all using
the same oracle methodology (GPT-4.1 with full ontology context injection)
described in Appendix~\ref{app:generation}. For each scene, a natural-language
manipulation query is generated and paired with a gold
\texttt{BehaviorTree.CPP\,v4} XML BT derived from the gold priority order
$\pi^*$. This yields \textbf{424 (NL-query, gold-BT) training pairs}, held to
strict non-overlap with the 82 novel test samples.

Gold BTs are generated by conditioning GPT-4.1 on the gold priority order
$\pi^*$, the scene description $\{O, C, P, R\}$, a fixed node-library
specification (action nodes: \texttt{Pick}, \texttt{Place},
\texttt{CheckStackSafety}, \texttt{CheckSurface}; control-flow:
\texttt{Sequence}, \texttt{Fallback}, \texttt{ReactiveFallback}), and two-shot
BT examples. Stack-dependency handling is enforced: any \texttt{on\_top\_of}
relation activates a \texttt{CheckStackSafety} subtree, which \sdhr measures.

\subsection{\slora mounting strategies: Instruct and Matched}

After \slora training, it is unclear whether to mount the adapter onto the
unmodified KI model or onto a freshly initialised variant. Motivated by the
disentanglement principle of Eq.~\ref{eq:domainbt}, we evaluate two strategies:

\begin{description}[leftmargin=1em,itemsep=1pt]
  \item[\textbf{Instruct}] \slora is trained using the instruct parent model as
    initialisation (not any KI-variant), then mounted onto each KI-variant via
    task arithmetic. This cleanly separates knowledge and skill: the skill
    adapter has never seen \orderworld domain physics, but the combined model
    expresses both.
  \item[\textbf{Matched}] \slora is trained using the specific KI-variant as
    initialisation, then mounted back onto that same model. This can capture
    cross-objective synergies but risks entangling knowledge and skill signals,
    producing destructive interference in the learning subspaces.
\end{description}

\subsection{Results on 82 novel samples}
\label{sec:slora_results}

Tables~\ref{tab:slora_instruct} and~\ref{tab:slora_matched} report \slora
results for all three model families. \tauorder is the primary metric; \parcmd
and P-PAR are complementary diagnostics. Additional patterns from the full
record are deferred to Appendix~\ref{app:slora_full}.

\begin{table}[t]
  \centering
  \caption{\textbf{\slora results - Instruct mounting} (82 novel samples).
    Best per family in \textbf{bold}. P-PAR: $n$/82 exactly correct sequences.}
  \label{tab:slora_instruct}
  \setlength{\tabcolsep}{3.5pt}
  \renewcommand{\arraystretch}{1.05}
  \footnotesize
  \begin{tabular}{llccccc}
    \toprule
    \textbf{Fam.} & \textbf{KI Var.} & \svr & \tauorder & \parcmd & P-PAR & \sdhr \\
    \midrule
    \multirow{8}{*}{L-8B}
      & Inst.$+$\slora & 100.0 & 0.830 & 0.629 & 36 & 100.0 \\
      & Normal         & 100.0 & 0.799 & 0.675 & 43 & 100.0 \\
      & Overfit        & 100.0 & 0.813 & 0.702 & 46 & 100.0 \\
      & \dkl           & 100.0 & 0.815 & 0.718 & 46 & 100.0 \\
      & RAFT           & 100.0 & \textbf{0.833} & \textbf{0.743} & \textbf{50} & 100.0 \\
      & Unified        &  98.8 & \textbf{0.848} & 0.724 & 49 &  98.3 \\
      & SSR-CLM        & \multicolumn{5}{c}{$\svr = 0.0$ (collapsed)} \\
      & 12-,\,21-Comb. & \multicolumn{5}{c}{$\svr = 0.0$ (collapsed)} \\
    \midrule
    \multirow{4}{*}{M-7B}
      & Inst.$+$\slora &  98.8 & 0.762 & 0.560 & 32 & 100.0 \\
      & RAFT           &  98.8 & \textbf{0.826} & \textbf{0.708} & \textbf{38} &  98.3 \\
      & \dkl           &  32.9 & 0.262 & 0.181 &  9 &  32.8 \\
      & Normal,\,Overfit & \multicolumn{5}{c}{$\svr = 0.0$ (collapsed)} \\
    \midrule
    \multirow{5}{*}{Q-4B}
      & Normal         &  98.8 & 0.704 & 0.510 & 29 &  98.3 \\
      & Overfit        & 100.0 & \textbf{0.760} & \textbf{0.480} & \textbf{22} & 100.0 \\
      & \dkl           &  98.8 & 0.717 & 0.399 & 16 &  91.4 \\
      & Unified        &  89.0 & 0.663 & 0.413 & 22 &  50.0 \\
      & 12-,\,21-Comb. & \multicolumn{5}{c}{$\svr = 0.0$ (collapsed)} \\
    \midrule
    \multicolumn{2}{l}{GPT-4.1 (Vanilla)} & 96.3 & 0.480 & 0.350 & 22 & 100.0 \\
    \bottomrule
  \end{tabular}
\end{table}

\begin{table}[t]
  \centering
  \caption{\textbf{\slora results - Matched mounting} (82 novel samples).
    Matching the \slora initialisation to the KI-variant benefits LLaMA
    (Overfit/SSR-CLM) and uniformly stabilises Qwen; most Mistral KI-variants
    collapse.}
  \label{tab:slora_matched}
  \setlength{\tabcolsep}{3.5pt}
  \renewcommand{\arraystretch}{1.05}
  \footnotesize
  \begin{tabular}{llccccc}
    \toprule
    \textbf{Fam.} & \textbf{KI Var.} & \svr & \tauorder & \parcmd & P-PAR & \sdhr \\
    \midrule
    \multirow{8}{*}{L-8B}
      & Normal    & 100.0 & 0.781 & 0.649 & 40 & 100.0 \\
      & Overfit   & 100.0 & \textbf{0.848} & \textbf{0.736} & 46 & 100.0 \\
      & \dkl      & 100.0 & 0.831 & 0.697 & 44 &  96.6 \\
      & RAFT      & 100.0 & 0.767 & 0.574 & 32 &  98.3 \\
      & SSR-CLM   &  97.6 & 0.844 & 0.705 & 45 &  91.4 \\
      & Unified   &  75.6 & 0.601 & 0.497 & 31 &  74.1 \\
      & 12-Comb.  &  95.1 & 0.806 & 0.695 & \textbf{47} &  93.1 \\
      & 21-Comb.  &  93.9 & 0.791 & 0.668 & 42 &  89.7 \\
    \midrule
    \multirow{4}{*}{M-7B}
      & Normal    & 100.0 & \textbf{0.816} & \textbf{0.580} & \textbf{33} &  96.6 \\
      & RAFT      &  98.8 & 0.782 & 0.554 & 32 &  91.4 \\
      & \dkl      &  29.3 & 0.270 & 0.219 & 14 &  25.9 \\
      & Others    & \multicolumn{5}{c}{$\svr = 0.0$ (collapsed)} \\
    \midrule
    \multirow{6}{*}{Q-4B}
      & Normal    & 100.0 & 0.796 & 0.645 & \textbf{41} & 100.0 \\
      & Overfit   & 100.0 & 0.789 & 0.641 & 37 & 100.0 \\
      & \dkl      & 100.0 & \textbf{0.813} & 0.629 & 36 &  96.6 \\
      & RAFT      & 100.0 & 0.803 & \textbf{0.682} & 46 & 100.0 \\
      & Unified   & 100.0 & 0.810 & 0.630 & 37 &  98.3 \\
      & 12-,\,21-Comb. & \multicolumn{5}{c}{$\svr = 0$--$10\%$ (almost collapsed)} \\
    \bottomrule
  \end{tabular}
\end{table}

\paragraph{LLaMA-3.1-8B is the strongest family, and its top performers reveal
two distinct pathways to BT quality}
Under Instruct mounting, LLaMA-Unified and LLaMA-RAFT achieve
$\tauorder = 0.848$ and $0.833$ respectively. Under Matched mounting,
LLaMA-Overfit and SSR-CLM reach $0.848$ and $0.844$. These are the highest
results across all families and configurations, but they arise through two
mechanistically distinct routes.

\emph{Pathway~1 - world-model induction (Unified).}
LLaMA-Unified achieved the highest \orderspatial \tauorderu of $0.81$ with
near-equal seen/unseen gains of $+0.26$ each, the strongest evidence for
abstract-layer physics internalisation over memorisation in
Section~\ref{sec:orderspatial_results}. This compositional generalisation
capacity appears to be exactly what \slora then leverages: a model that can
freely compose the interaction grammar to produce a correct priority ordering
in novel scenes transfers that ability naturally into BT structure, where the
ordering must be expressed as an executable tree. Unified also exhibits the
sharpest Instruct-vs-Matched asymmetry ($\Delta = 0.247$), confirming that
clean disentanglement of knowledge and skill is essential precisely for the
variant whose knowledge representation is most compositionally structured.

\emph{Pathway~2 - deep parametric encoding (Overfit).}
LLaMA-Overfit's \orderspatial unseen \tauorder is only $0.54$, barely above the
random baseline of $0.500$. Overfit does \emph{not} demonstrate world-model
induction in the \orderspatial sense; it extended training to near-zero loss,
which produced a model that closely reproduced seen-scene analyses without
meaningful generalisation. Yet under Matched \slora mounting it ties for the
highest BT \tauorder ($0.848$). Our hypothesis is that this is structural: deep,
near-lossless encoding drives domain physics patterns into the model's
parametric weights at a representation level that, while not compositionally
generalisable at the scene-analysis level, provides a strong prior for the
BT-generation skill to exploit. \slora Matched training from this initialisation
finds a narrow but stable optimum that produces correctly ordered trees - a
different mechanism from Unified's compositional transfer, and one that only
emerges under Matched mounting.

Meanwhile, LLaMA-SSR-CLM, which achieved the highest MCQ aggregate ($78.6\%$)
and strong \orderspatial performance ($0.76$ seen, $0.74$ unseen), collapses
completely under Instruct mounting ($\svr = 0$), a failure mode already flagged
in Section~\ref{sec:ppl_diss} as undetectable from PPL or MCQ scores. Under
Matched mounting it recovers to $\tauorder = 0.844$, confirming the failure is a
format-interference artefact of the specific adapter composition, not a
fundamental limitation of SSR-CLM's knowledge representation.

\paragraph{RAFT transfers well to BT generation despite moderate MCQ scores}
LLaMA-RAFT achieves $\tauorder = 0.833$ under Instruct mounting despite having
only a $47.2\%$ MCQ aggregate (against a $34.7\%$ zero-shot baseline). RAFT's
retrieval-augmented fine-tuning objective, which trains the model to reason
faithfully over retrieved context, appears to transfer directly to the
downstream task of conditioning on structured inputs (scene JSON $+$ priority
list) and generating structured outputs (XML). The same structured
input--output conditioning that RAFT was trained on in Stage~1 is exactly what
\slora requires in Stage~2, making RAFT a particularly compatible
initialisation. For Mistral, RAFT with Instruct mounting likewise performs
best, positively correlating with its \orderspatial performance.

\paragraph{SFT alone is a surprisingly strong baseline}
The instruct model equipped only with the \slora adapter
(Inst.$+$\slora, no KI knowledge) achieves $\tauorder = 0.830$ for LLaMA and
$0.762$ for Mistral. For LLaMA this falls only $0.018$ below the
LLaMA-Unified / Instruct peak ($0.848$). KI knowledge therefore provides a
meaningful but not dramatic advantage on BT format quality as measured by
\tauorder. It is likely that the semantic correctness of the ordering, which
the gold BTs guarantee, is where KI knowledge matters most, and that this
advantage is only partially captured by \tauorder and \parcmd. Determining the
precise conditions under which domain knowledge ingestion provides decisive
gains over pure skill fine-tuning remains an important direction for future
work.

\paragraph{Mistral-7B is brittle to adapter composition, but RAFT is the
exception that closes the knowledge-to-action chain}
Several Mistral KI-variants produce $\svr = 0\%$ even after \slora mounting,
with failures concentrated in variants (Overfit, SSR-CLM, Unified,
12-/21-Combined) that underwent the most aggressive CPT objectives. This
mirrors the Mistral-Unified collapse of Section~\ref{sec:ppl_diss}
($S_\mathrm{agg} = 1.8\%$, undetected by PPL $= 2.15$) and reflects Mistral's
parameter geometry being particularly sensitive to directional interference
from multiple LoRA objectives composed via task arithmetic. Under Matched
mounting, Mistral-Normal achieves $\tauorder = 0.816$, the best Matched variant
for this family. However, the best Mistral result \emph{overall} is
Mistral-RAFT under Instruct mounting ($\tauorder = 0.826$, $\svr = 98.8\%$),
and this is not coincidental: Mistral-RAFT was the only Mistral variant to
exhibit genuine world-model induction evidence in
Section~\ref{sec:orderspatial_results}, an inverted seen/unseen \orderspatial
gap with unseen improvement ($+0.21$) exceeding seen improvement ($+0.18$).
For Mistral, the knowledge-to-action chain connects precisely through RAFT: the
one variant whose benchmark profile reflects abstract internalisation is also
the one that best converts that knowledge into executable plans, and is robust
enough under Instruct mounting to avoid the collapse that afflicts all other
aggressive CPT variants in this family.

\paragraph{Qwen3-4B performs well under Matched mounting but is sensitive to
adapter initialisation}
Under Matched mounting, all five Qwen variants achieve $\svr = 100\%$ with
$\tauorder \in [0.789, 0.813]$, a remarkably uniform and robust band. Under
Instruct mounting, however, the base Instruct $+$ \slora collapses
($\svr = 0\%$) despite Qwen-Instruct having $97.6\%$ \svr in zero-shot
evaluation. This is a sharp architecture-specific sensitivity: the Qwen3 hybrid
reasoning architecture requires that the \slora adapter be trained from a model
already exposed to \orderworld physics. Qwen-\dkl, which had the highest MCQ
aggregate ($86.3\%$), achieves a solid but not dominant BT result
($\tauorder = 0.813$ under Matched), while Qwen-Unified, which had the
strongest Qwen spatial MCQ ($80.6\%$), achieves $0.810$. For Qwen, the ranking
of KI variants in BT generation is not strongly correlated with MCQ rankings,
confirming the MCQ--BT dissociation at the intra-family level as well.

\subsection{What \orderbench predicts, and what it does not}
\label{sec:predictors}

This subsection is the benchmark-facing payoff of the downstream experiments:
it establishes which \order measurement actually anticipates robot plan
quality.

\textbf{MCQ aggregate is an unreliable predictor.}
LLaMA-Unified achieves among the strongest \slora results
($\tauorder = 0.848$ under Instruct) while ranking only moderately on
\orderbench MCQ aggregate ($72.8\%$, third among LLaMA variants, behind SSR-CLM
at $78.6\%$ and 12-Combined at $78.1\%$), but the best on \orderspatial
($\tauorderu = 0.81$). Conversely, LLaMA-SSR-CLM achieves the highest MCQ
aggregate yet collapses under Instruct \slora ($\svr = 0$), and LLaMA-RAFT
achieves only $47.2\%$ MCQ aggregate yet reaches $\tauorder = 0.833$. A
practitioner selecting models by MCQ rank would discard RAFT and retain
SSR-CLM: the opposite of the correct deployment decision.

\textbf{\orderspatial concordance is a much stronger predictor, with
principled exceptions.}
It cleanly identifies the strongest performer (Unified,
$\tauorderu = 0.81 \rightarrow$ BT $0.848$) and anticipates the catastrophic
failure case: even SSR-CLM's strong \orderspatial score of $0.74$ cannot
overcome a format-interference failure mode under Instruct composition, so
\orderspatial correctly flags it as a high-knowledge model but cannot predict
the interaction between CPT objectives and \slora adapter geometry. Where
\orderspatial is less predictive is among intermediate variants: LLaMA-RAFT
($\tauorderu = 0.60$) outperforms LLaMA-\dkl ($0.58$) and LLaMA-Overfit
($0.54$) on \orderspatial only marginally, yet the BT gap between them is not
proportional, and Overfit ties for the top BT result under Matched mounting.
The dissociation is principled: \orderspatial measures compositional
generalisation of the \emph{priority-reasoning procedure} through free-form
text, while BT generation additionally requires structured XML generation skill
and scene-to-tree mapping - contributions that \slora supplies and that
interact differently with each KI variant's weight geometry. The appropriate
interpretation is therefore that \orderspatial concordance \emph{shortlists top
candidates and rules out failures}, while the full BT ranking requires
empirical post-merge \svr validation - precisely the discipline that
Section~\ref{sec:ppl_diss} recommends for PPL-based model selection.

\section{Retrieval for Fine-Tuned SLMs: Context Hijacking and
  Two-Stage Recovery}
\label{sec:rag}

\subsection{Context hijacking}
\label{sec:hijacking}

Given that retrieval nearly saturates \orderbench MCQ
(Section~\ref{sec:zeroshot}), a natural question is whether the same structured
retrieval index can augment fine-tuned SLMs at BT-generation time. The answer,
in its unadapted form, is no, and the reason is architectural: \emph{context
hijacking}.

The structured retrieval index (Appendix~\ref{app:rag_bt}) assembles rich
analytical physics descriptions within a 4{,}000-token budget. When a
fine-tuned SLM receives this context directly in the BT-generation call, it
attends to those descriptions and produces the kind of free-form reasoning text
that the CPT training rewarded, rather than the XML output that \slora trained.
The model's CPT and \slora objectives are in conflict: CPT trained it to
produce structured scene analyses; the injected context re-activates that
objective and suppresses the \slora-trained XML generation mode. This is a
direct instance of the lost-in-the-middle
phenomenon~\citep{liu2024lost}, in which extended context saturates effective
utilisation and degrades instruction compliance; GPT-4.1's much larger context
window makes it resilient to the same load, which is why the failure is
specific to the fine-tuned SLMs. Table~\ref{tab:context_hijack} quantifies the
collapse and its recovery.

\begin{table}[t]
  \centering
  \caption{\textbf{Context-hijacking failure and recovery} (LLaMA-3.1-8B).
    Injecting the full assembled ontology context directly into the
    BT-generation call collapses \svr from 100\% to 0\% in the fine-tuned SLM;
    the two-stage pipeline with a condensed ($\leq$800-token)
    priority-reasoning call restores functional output, though still below the
    closed-book peak.}
  \label{tab:context_hijack}
  \setlength{\tabcolsep}{5pt}
  \renewcommand{\arraystretch}{1.1}
  \footnotesize
  \begin{tabular}{lcc}
    \toprule
    \textbf{Setup} & \svr & \tauorder \\
    \midrule
    Closed-book best (RAFT, Instruct)          & 100.0\% & 0.833 \\
    SLM $+$ full assembled context (unadapted) &   0.0\% & 0.000 \\
    SLM $+$ two-stage-cs (adapted)             & 100.0\% & 0.798 \\
    \bottomrule
  \end{tabular}
\end{table}

\subsection{Structured retrieval for SLMs}
\label{sec:ragv3_desc}

The retrieval framework maintains separate exact-match stores for intrinsic
semantics, pairwise interactions, and counterfactuals, and a FAISS vector index
(all-MiniLM-L6-v2) for higher-order regimes, dominance axioms, and QnA entries.
Retrieved chunks are ranked and assembled within a 4{,}000-token budget using
the knowledge-type weighting of Table~\ref{tab:rag_weights}
(Appendix~\ref{app:rag_bt}). Two training-set examples are retrieved per test
scene by structural similarity (object count, stacking presence,
duplicate-object presence). Spatial scene descriptions are deliberately
excluded from the retrieved examples (\emph{no-spatial} mode), a configuration
that consistently outperformed pipelines including spatial scene text.

\subsection{Two pipeline variants}
\label{sec:four_pipelines}

Context hijacking motivates the two structured pipeline variants of
Figure~\ref{fig:rag_pipelines}, which distribute the priority-reasoning burden
between retrieval and the SLM without ever handing the SLM a full ontology
dump.

\begin{description}[leftmargin=1em,itemsep=2pt]
  \item[\textbf{Two-Stage}] The SLM itself performs Stage~1: it receives a
    condensed ontology context ($\leq$800 tokens, top-5 counterfactual, axiom,
    and inter-object pairwise chunks) and outputs a numbered priority list. A
    second SLM call generates the BT from that list. The two-stage separation
    prevents the full ontology context from saturating the BT-generation call.
  \item[\textbf{Two-Stage-CS}] Combines the two-stage separation with
    \emph{constrained start}: the first
    \texttt{<root BTCPP\_format="4">} tokens are forced into the decoder prefix
    of the Stage~2 call before autoregressive generation begins. This commits
    the model to valid XML from token~1, preventing the most common failure
    mode - the model opening with analysis text rather than an XML tag.
\end{description}

\begin{figure}[t]
\begin{tcolorbox}[
  colback=gray!4, colframe=NavyBlue!60,
  title={\small\textbf{Two-Stage Pipeline}},
  fontupper=\scriptsize, boxrule=0.7pt,
  left=1.5mm, right=1.5mm, top=1.5mm, bottom=1mm]
\begin{algorithmic}[1]
  \State \textbf{Input:} scene $\{O,C,P,R\}$, NL task $q$
  \State Retrieve weighted chunks $\mathcal{H}$; condense to $\mathcal{H}_{800}$
         (top-5 CF / axiom / inter-pairwise)
  \State \textbf{Stage 1:} $\pi \leftarrow
         \textsc{SLM}(\text{sys}_{\pi},\, \mathcal{H}_{800},\, \{O,C,P,R\})$
  \State \textbf{Stage 2:} $\hat{t} \leftarrow
         \textsc{SLM}(\text{sys}_\text{BT},\, q,\, \{O,C,P,R\},\, \pi)$
  \State \textbf{Return:} BT XML $\hat{t}$
\end{algorithmic}
\end{tcolorbox}
\vspace{-4pt}
\begin{tcolorbox}[
  colback=gray!4, colframe=NavyBlue!60,
  title={\small\textbf{Two-Stage-CS Pipeline (constrained start)}},
  fontupper=\scriptsize, boxrule=0.7pt,
  left=1.5mm, right=1.5mm, top=1.5mm, bottom=1mm]
\begin{algorithmic}[1]
  \State \textbf{Input:} scene $\{O,C,P,R\}$, NL task $q$
  \State Retrieve $\mathcal{H}$; condense to $\mathcal{H}_{800}$
  \State \textbf{Stage 1:} $\pi \leftarrow
         \textsc{SLM}(\text{sys}_{\pi},\, \mathcal{H}_{800},\, \{O,C,P,R\})$
  \State Force prefix $p_0 \leftarrow$ \texttt{<root BTCPP\_format="4">}
  \State \textbf{Stage 2:} $\hat{t} \leftarrow p_0 \;\|\;
         \textsc{SLM}(\text{sys}_\text{BT},\, q,\, \{O,C,P,R\},\, \pi,\, p_0)$
  \State \textbf{Return:} BT XML $\hat{t}$
\end{algorithmic}
\end{tcolorbox}
\caption{\textbf{The two structured retrieval pipelines.} Both delegate
  Stage~1 priority reasoning to a condensed SLM call ($\leq$800 tokens) rather
  than injecting the full assembled context into the generation call; the CS
  variant additionally forces an XML prefix to prevent prose-opening failures.
  Both share the same weighted retrieval index and no-spatial few-shot
  conditioning.}
\label{fig:rag_pipelines}
\end{figure}

\subsection{Full grammar-constrained decoding degrades performance}
\label{sec:cfd}

A natural extension of prefix forcing is to enforce the full
\texttt{BehaviorTree.CPP\,v4} XML grammar at every decoding step via
CFG-guided generation~\citep{zhang2023don}. We evaluate this ``hard'' constraint
against the softer constrained start. Full CFG-guided decoding \emph{degrades}
BT generation quality rather than improving it. The failure mode is consistent:
at each generation step the grammar constraint prunes the valid token set to a
small admissible vocabulary; the SLM's probability mass, shaped by \slora
fine-tuning, is forced onto tokens that are grammatically valid but
semantically inconsistent with the intended plan, producing structurally
compilable but semantically incorrect BTs. Constrained-start prefix forcing
provides the effective balance: it eliminates the most common failure mode
(prose opening) while leaving the model's fine-tuned generation distribution
undisturbed at subsequent steps.

\subsection{Structured retrieval results}
\label{sec:ragresults}

Table~\ref{tab:ragv3} reports the two pipelines across the three families, with
the closed-book \tauorder for the same (variant, mounting) pair as reference.

\begin{table}[t]
  \centering
  \caption{\textbf{Structured retrieval results} (82 novel samples).
    CB-\tauorder is the closed-book reference for the same (variant, mounting)
    pair. Retrieval revives variants that collapsed closed-book, but does not
    reach the closed-book peak of $0.848$.}
  \label{tab:ragv3}
  \setlength{\tabcolsep}{3pt}
  \renewcommand{\arraystretch}{1.05}
  \footnotesize
  \begin{tabular}{llcccc}
    \toprule
    \textbf{Variant / Mount} & \textbf{Pipeline} & \svr & \tauorder
      & \parcmd & CB-\tauorder \\
    \midrule
    \multicolumn{6}{l}{\textit{LLaMA-3.1-8B}} \\
    Base-Inst / Inst. & two-stage    & 100\% & 0.819 & 0.686 & 0.830 \\
    Base-Inst / Inst. & two-stage-cs & 100\% & 0.798 & 0.663 & 0.830 \\
    Overfit / Match.  & two-stage    & 100\% & 0.811 & 0.713 & 0.848 \\
    Overfit / Match.  & two-stage-cs &  99\% & 0.801 & 0.714 & 0.848 \\
    12-Comb. / Match. & two-stage    & 100\% & 0.824 & 0.719 & 0.806 \\
    12-Comb. / Match. & two-stage-cs & 100\% & \textbf{0.843} & \textbf{0.723} & 0.806 \\
    \midrule
    \multicolumn{6}{l}{\textit{Mistral-7B}} \\
    Normal / Inst.    & two-stage    & 95.1\% & 0.748 & 0.535 & N/A$^{\dagger}$ \\
    \midrule
    \multicolumn{6}{l}{\textit{Qwen3-4B}} \\
    12-Comb. / Match. & two-stage-cs & 100\% & \textbf{0.796} & \textbf{0.646} & $0$--$10\%$ \svr$^{\ddagger}$ \\
    Base-Inst / Match.& two-stage-cs & 100\% & 0.720 & 0.530 & N/A$^{\dagger}$ \\
    \midrule
    \multicolumn{6}{l}{\textit{GPT-4.1 frontier reference}} \\
    GPT-4.1           & two-stage    & -   & 0.599 & -   & 0.480 \\
    GPT-4.1           & two-stage-cs & -   & 0.606 & -   & 0.480 \\
    \bottomrule
  \end{tabular}
  \vspace{2pt}
  \begin{minipage}{0.98\columnwidth}
  \scriptsize $^{\dagger}$Closed-book collapse ($\svr = 0$) for this
  (variant, mounting) pair. $^{\ddagger}$Qwen 12-/21-Combined were almost
  collapsed closed-book under Matched ($\svr = 0$--$10\%$).
  \end{minipage}
\end{table}

\textbf{Retrieval revives collapsed variants but cannot match closed-book
peaks.} The most striking result is for the Combined variants. Qwen's
12-Combined was almost collapsed in closed-book Matched evaluation
($\svr = 0$--$10\%$), yet under two-stage-cs it reaches $\tauorder = 0.796$
with $\svr = 100\%$; LLaMA's 12-Combined, which collapses entirely under
Instruct mounting, reaches $\tauorder = 0.843$ under two-stage-cs, the
strongest retrieval-augmented result overall. Delegating priority resolution to
a separate, condensed call removes enough of the compositional reasoning demand
from the generation call that variants which cannot sustain it alone still
produce valid and well-ordered trees.

\textbf{The best retrieval result does not exceed the best closed-book.}
LLaMA 12-Combined / two-stage-cs ($0.843$) approaches but does not reach the
closed-book peak of LLaMA-Unified / Instruct ($0.848$). This gap is principled:
parametric knowledge internalised via CPT provides a compositional reasoning
capacity that retrieved context cannot fully replicate, because composing
multiple interaction rules requires the grammar to be available as a parametric
procedure rather than as retrieved text. This is precisely the retrieval
dissociation established on \orderspatial in
Section~\ref{sec:orderspatial_results}, now re-confirmed in the downstream BT
generation task. Finding the optimal retrieval architecture for fine-tuned,
knowledge-ingested SLMs on this task remains an open problem.

\textbf{GPT-4.1 benefits from the condensed pipelines, but far less than
adaptation benefits the SLMs.} GPT-4.1 improves from $\tauorder = 0.480$
(vanilla) to $0.599$ (two-stage) and $0.606$ (two-stage-cs, P-PAR $= 10$/82),
which we take as its strongest configuration on this task and use as the
frontier reference in Section~\ref{sec:vs_gpt}.

\textbf{\svr instability is model- and pipeline-specific.} LLaMA variants
maintain $\svr \approx 100\%$ across both pipelines. Mistral is the least
stable family, with Normal / two-stage at $95.1\%$ the only surviving
configuration, the same tokenisation and interference sensitivity that caused
systematic instability across Mistral's aggressive CPT variants in
Sections~\ref{sec:ppl_diss} and~\ref{sec:slora_results}. Qwen base-instruct
achieves $\svr = 100\%$ under two-stage-cs, confirming the
architecture-specific pipeline sensitivity observed in closed-book Instruct
mounting.

\section{Human-in-the-Loop Execution on \texttt{iiwa7} in Gazebo}
\label{sec:hitl}

\subsection{Simulation environment and visual grammar testbed}

Experiments are conducted in Gazebo (ROS~1 Noetic) using an open-source
\texttt{iiwa7} arm model. The tabletop workspace contains primitively coloured
geometric objects (cubes, cylinders, spheres, cones, discs) in colours and
textures drawn from the five-axis visual grammar $\mathcal{G}$, enabling direct
correspondence between the Gazebo scene and the \orderworld evaluation world.
Object poses are published via \texttt{/scene\_perception} as structured JSON
$\{O, C, P, R\}$. The BT executor is a ROS node consuming XML BTs via an action
server, with MoveIt providing motion planning and collision avoidance.

\subsection{End-to-end pipeline}
\label{sec:e2e}

The end-to-end pipeline begins with the VLM of Section~\ref{sec:perception},
which perceives the tabletop scene and converts it into the structured JSON
representation $\{O, C, P, R\}$. Given this scene description and a
natural-language query (e.g.\ \textit{``Clear the workspace following safety
protocols''}), the Domain-BT-LM generates an executable
\texttt{BehaviorTree.CPP\,v4} XML encoding the intended priority ordering
$\hat{\pi}$. A human operator remains in the loop to verify the generated BT
and can iteratively refine it via re-prompting (Section~\ref{sec:reprompt})
until satisfactory. Once approved, the BT is parsed and executed by the
\texttt{BehaviorTree.CPP} engine. The final tree ensures that all objects are
manipulated and placed into their respective colour-coded bins while adhering
to the safest priority order defined by \orderworld.
Figure~\ref{fig:gazebo} shows a snapshot of the three-layer architecture in
action on Configuration~\#88.

\begin{figure*}[t]
  \centering
  \includegraphics[width=\textwidth]{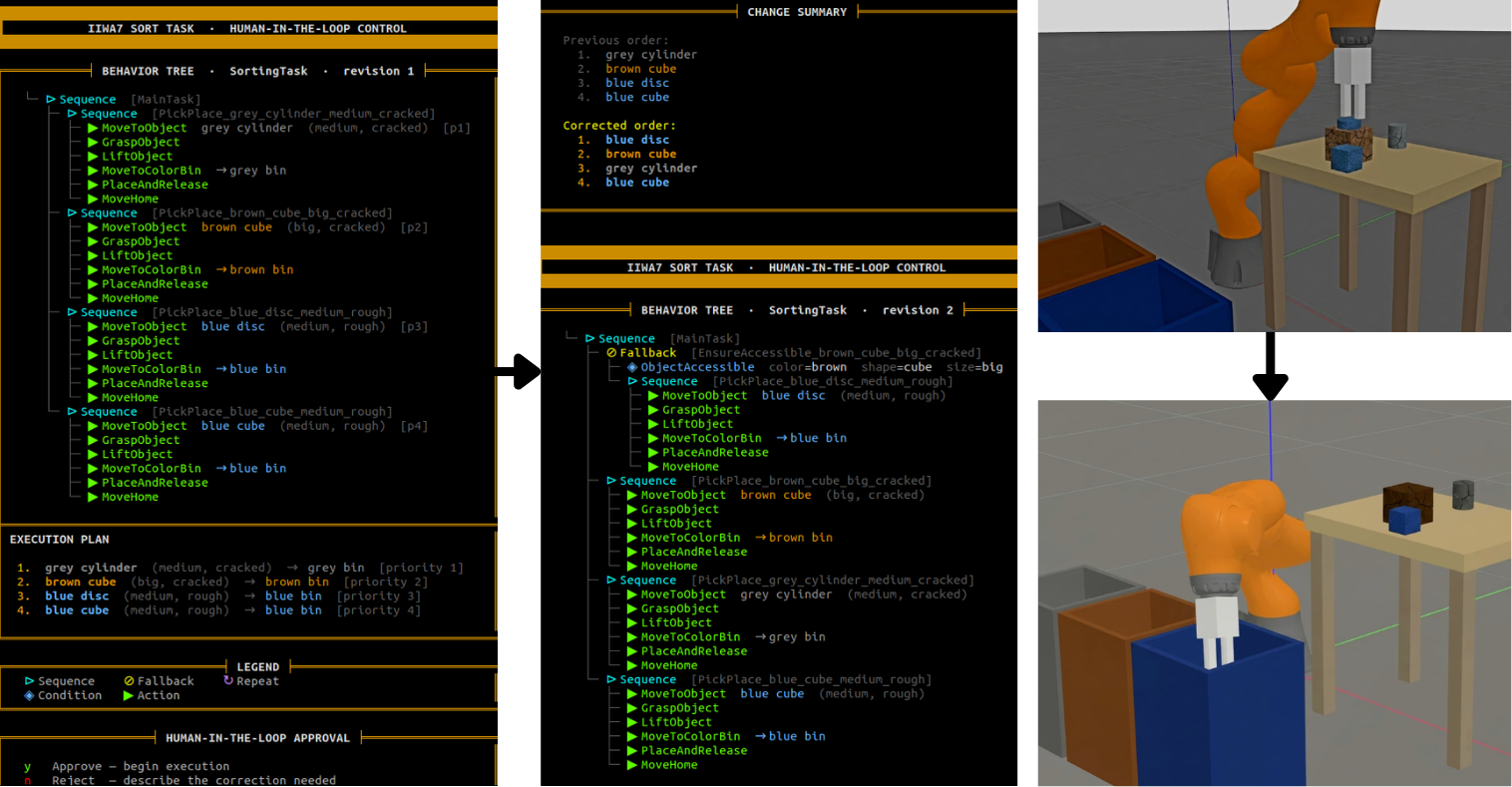}
  \caption{%
    \textbf{End-to-end pipeline snapshots on the \texttt{iiwa7} in Gazebo.}
    \textbf{Left:} the direct LLM-generated BT XML and execution summary for
    Configuration~\#88, given the VLM $\{O, C, P, R\}$ JSON and the NL user
    input; the tree commits two errors, placing the grey cylinder at
    priority~1 (correct: priority~3) and emitting the blue disc as a standalone
    top-level sequence instead of inside the required \texttt{Fallback}
    accessibility guard.
    \textbf{Centre:} the corrected BT XML after a single HITL operator
    re-prompt, in which \texttt{EnsureAccessible\_brown\_cube} becomes the
    first child of \texttt{MainTask} and the blue disc appears only as the
    conditional blocker-removal step inside it.
    \textbf{Top right:} the initial Gazebo scene for Configuration~\#88.
    \textbf{Bottom right:} the \texttt{iiwa7} arm grasping the first-priority
    object (the blue disc, the stacking blocker) and placing it in the blue
    colour bin.}
  \label{fig:gazebo}
\end{figure*}

\subsection{Re-prompting for runtime priority correction}
\label{sec:reprompt}

A key design principle of the HITL framework is that operator corrections
should require \emph{minimal effort} and produce \emph{fully executable}
corrected trees without re-running the full pipeline. When the Domain-BT-LM
generates a syntactically valid BT with an incorrect priority order (e.g.\
objects A, B, C ordered A$\to$C$\to$B when the correct order is A$\to$B$\to$C),
the operator issues a single correction. It was qualitatively observed that
even minor syntactic errors are corrected in the same pass, as in the
Configuration~\#88 example.

The Domain-BT-LM receives the original scene description, the original BT as
context, and the correction prompt. It regenerates only the ordering structure
while preserving action node semantics and stack-safety subtrees. In our
evaluations, all single-correction HITL cases produced fully compilable,
correctly ordered trees on the first re-prompt, demonstrating that
\slora-trained models have robust format adherence under correction
instructions. Figure~\ref{fig:hitl_example} traces both errors in the
LLM-generated tree for Configuration~\#88 and their resolution via a single
operator re-prompt. The re-prompt template is provided in Appendix~\ref{app:hitl-prompt}.

\begin{figure*}[t]
\centering
\begin{tcolorbox}[
  colback=NavyBlue!3, colframe=NavyBlue!50,
  title={\small\textbf{%
    Scene - Configuration~\#88 \quad
    (\#1~grey cylinder,\;\#2~blue disc,\;\#3~brown cube,\;\#4~blue cube)}},
  fontupper=\normalfont\scriptsize,
  boxrule=0.65pt, left=3mm, right=3mm, top=1.5mm, bottom=1.5mm]
\textbf{Objects:}\enskip
(1)~grey cylinder (medium, cracked);\quad
(2)~blue disc (medium, rough);\quad
(3)~brown cube (big, cracked);\quad
(4)~blue cube (medium, rough).\\[2pt]
\textbf{Spatial:}\enskip
\texttt{disc\;stacked\_with\;brown\_cube}~$\bullet$~%
\texttt{brown\_cube\;clustered\_with\;blue\_cube}~$\bullet$~%
\texttt{grey\_cylinder\;isolated\_from\;all}.\\[2pt]
\textbf{Key hazard:} the blue disc (\#2) sits \emph{on top of} the brown cube
(\#3), blocking access; the disc must be cleared before the brown cube can be
grasped.\\[2pt]
\textbf{\orderworld priority:}\enskip
\textit{blue disc}~$\to$~\textit{brown cube}~$\to$~\textit{grey
cylinder}~$\to$~\textit{blue cube.}
\end{tcolorbox}


\begin{minipage}[t]{0.483\textwidth}
\begin{tcblisting}{%
  height=7.35cm,
  valign=top,
  colback=red!3!white, colframe=red!40!gray,
  title={\small\textbf{LLM-Generated BT \quad(Incorrect)}},
  boxrule=0.65pt, left=1.5mm, right=1.5mm, top=1.5mm, bottom=1.5mm,
  listing only,
  listing options={
    language=XML,
    basicstyle=\ttfamily\tiny,
    columns=fullflexible,
    keepspaces=true,
    showstringspaces=false,
    commentstyle=\color{gray!65},
    keywordstyle=\color{NavyBlue!75},
    stringstyle=\color{purple!55!black},
    escapechar=|,
    morekeywords={root,BehaviorTree,Sequence,Action,Fallback,Condition}
  }
}
<root BTCPP_format="4">
 <BehaviorTree ID="SortingTask">
  <Sequence name="MainTask"> |{\normalfont\tiny\textcolor{red!75!black}{\bfseries$\leftarrow$ 4 flat Seqs; no Fallback guard}}|
   <Sequence name="PP_grey_cylinder_med_cracked"> |{\normalfont\tiny\textcolor{red!75!black}{\bfseries$\leftarrow$ wrong P1}}|
    <Action ID="MoveToObject" color="grey"
            shape="cylinder" priority="1"/>
    <Action ID="GraspObject"/>
    <Action ID="LiftObject"/>
    <Action ID="MoveToColorBin" color="grey"/>
    <Action ID="PlaceAndRelease"/>
    <Action ID="MoveHome"/>
   </Sequence>
   <Sequence name="PP_brown_cube_big_cracked">
    <!-- ... same 6 pick-place actions ... -->
   </Sequence>
   <Sequence name="PP_blue_disc_med_rough"> |{\normalfont\tiny\textcolor{red!75!black}{\bfseries$\leftarrow$ standalone task (error!)}}|
    <!-- ... same 6 pick-place actions ... -->
   </Sequence>
   <Sequence name="PP_blue_cube_med_rough">
    <!-- ... same 6 pick-place actions ... -->
   </Sequence>
  </Sequence>
 </BehaviorTree>
</root>

\end{tcblisting}
\end{minipage}%
\hfill%
\begin{minipage}[t]{0.483\textwidth}
\begin{tcblisting}{%
  colback=green!2!white, colframe=OliveGreen!55!gray,
  title={\small\textbf{Gold BT \quad(After HITL Correction)}},
  boxrule=0.65pt, left=1.5mm, right=1.5mm, top=1.5mm, bottom=1.5mm,
  listing only,
  listing options={
    language=XML,
    basicstyle=\ttfamily\tiny,
    columns=fullflexible,
    keepspaces=true,
    showstringspaces=false,
    commentstyle=\color{gray!65},
    keywordstyle=\color{OliveGreen!70!black},
    stringstyle=\color{OliveGreen!55!black},
    escapechar=|,
    morekeywords={root,BehaviorTree,Sequence,Action,Fallback,Condition}
  }
}
<root BTCPP_format="4">
 <BehaviorTree ID="SortingTask">
  <Sequence name="MainTask"> |{\normalfont\tiny\textcolor{OliveGreen!75!black}{\bfseries$\leftarrow$ 4 children: Fallback~$+$~3~Seqs}}|
   <Fallback name="EnsureAccessible_brown_cube_big_cracked"> |{\normalfont\tiny\textcolor{OliveGreen!75!black}{\bfseries$\leftarrow$ stack guard}}|
    <Condition ID="ObjectAccessible" color="brown"
               shape="cube" size="big"/>
    <Sequence name="PP_blue_disc_med_rough"> |{\normalfont\tiny\textcolor{OliveGreen!75!black}{\bfseries$\leftarrow$ blocker removal only}}|
     <Action ID="MoveToObject" color="blue"
             shape="disc" size="medium"/>
     <Action ID="GraspObject"/>
     <Action ID="LiftObject"/>
     <Action ID="MoveToColorBin" color="blue"/>
     <Action ID="PlaceAndRelease"/>
     <Action ID="MoveHome"/>
    </Sequence>
   </Fallback>
   <Sequence name="PP_brown_cube_big_cracked">
    <!-- ... same 6 pick-place actions ... -->
   </Sequence>
   <Sequence name="PP_grey_cylinder_med_cracked">
    <!-- ... same 6 pick-place actions ... -->
   </Sequence>
   <Sequence name="PP_blue_cube_med_rough">
    <!-- ... same 6 pick-place actions ... -->
   </Sequence>
  </Sequence>
 </BehaviorTree>
</root>
\end{tcblisting}
\end{minipage}

\vspace{3pt}

\begin{tcolorbox}[
  colback=gray!5, colframe=gray!60,
  title={\small\textbf{Operator Correction Prompt (HITL Re-Prompt Instance)}},
  fontupper=\ttfamily\scriptsize,
  boxrule=0.6pt, left=3mm, right=3mm, top=1.5mm, bottom=1mm]
\textrm{[HITL]}\; n Wrong order and wrong structure. The correct order is:
\texttt{Fallback(EnsureAccessible} for brown cube) first, then brown cube,
then grey cylinder, then blue cube. The Fallback's inner Sequence is the
ONLY place blue disc appears - it is a conditional recovery step, not a
standalone task. Remove any standalone blue disc Sequence from the top
level. \texttt{MainTask} must have exactly 4 children.
\end{tcolorbox}

\caption{%
  \textbf{Worked HITL correction example (Configuration~\#88).}
  \emph{Top:} the scene, with the blue disc stacked on the brown cube
  (blocking access), and the \order priority order.
  \emph{Left:} the LLM-generated tree commits two errors: (i)~the grey cylinder
  is placed at P1 instead of P3, and (ii)~the blue disc appears as a standalone
  top-level \texttt{Sequence} rather than inside the required
  \texttt{Fallback} guard, missing the \texttt{CheckStackSafety} structure
  entirely.
  \emph{Right:} the gold tree places \texttt{EnsureAccessible\_brown\_cube} as
  the first \texttt{Fallback} child of \texttt{MainTask}; the blue disc appears
  \emph{only} inside it as the conditional blocker-removal step, and the
  remaining three children follow the correct \order-derived priority order.
  \emph{Bottom:} the single operator correction prompt that drives the
  Domain-BT-LM to regenerate the gold tree in one pass. The operator either
  approves the BT with `y' or starts with `n' followed by the correction
  prompt.}
\label{fig:hitl_example}
\end{figure*}

\section{Closing the Knowledge-to-Action Chain}
\label{sec:vs_gpt}

Table~\ref{tab:vs_gpt} compares top-performing Domain-BT-LM variants against
GPT-4.1 references. All Domain-BT-LM configurations are \emph{closed-book}:
deployed without any retrieval at inference. All variants shown exceed the best
GPT-4.1 configuration ($\tauorder = 0.606$) by margins of $0.190$--$0.242$, and
the \tauorderu column makes the knowledge-to-action chain legible per family.

\begin{table}[t]
  \centering
  \caption{\textbf{Domain-BT-LMs vs.\ GPT-4.1} (82 novel samples). All
    Domain-BT-LM results are closed-book (no retrieval at inference); the
    GPT-4.1 reference uses its best structured pipeline (two-stage-cs).
    \tauorderu is the \orderspatial unseen concordance from
    Section~\ref{sec:orderspatial_results} (random baseline $= 0.500$),
    included to trace the knowledge-to-action chain.
    $^{\dagger}$SSR-CLM collapses under Instruct mounting ($\svr = 0$); the
    Matched result is shown. $^{\ddagger}$Mistral-RAFT/Inst.\ is the best
    Mistral result overall.}
  \label{tab:vs_gpt}
  \setlength{\tabcolsep}{2.8pt}
  \renewcommand{\arraystretch}{1.05}
  \footnotesize
  \begin{tabular}{llcccccc}
    \toprule
    \textbf{Fam.} & \textbf{KI / \slora} & \tauorderu & \svr & \sdhr
      & \tauorder & \parcmd & P-PAR \\
    \midrule
    \multicolumn{2}{l}{GPT-4.1 Vanilla}       & 0.44 & 96.3 & 100 & 0.480 & 0.420 & 22 \\
    \multicolumn{2}{l}{GPT-4.1 $+$ two-stage-cs} & -  & -  & - & 0.606 & -   & 10 \\
    \midrule
    \multicolumn{8}{l}{\textit{LLaMA-3.1-8B - Instruct \slora (closed-book)}} \\
    L-8B & Unified / Inst. & \textbf{0.81} & 98.8 & 98.3 & \textbf{0.848} & 0.724 & 49 \\
    L-8B & RAFT / Inst.    & 0.60 & 100  & 100  & 0.833 & \textbf{0.743} & \textbf{50} \\
    L-8B & \dkl / Inst.    & 0.58 & 100  & 96.6 & 0.815 & 0.718 & 46 \\
    L-8B & Overfit / Inst. & 0.54 & 100  & 100  & 0.813 & 0.702 & 46 \\
    \midrule
    \multicolumn{8}{l}{\textit{LLaMA-3.1-8B - Matched \slora (closed-book)}} \\
    L-8B & Overfit / Match. & 0.54 & 100  & 100  & \textbf{0.848} & \textbf{0.736} & 46 \\
    L-8B & SSR-CLM / Match.$^{\dagger}$ & 0.74 & 97.6 & 91.4 & 0.844 & 0.705 & 45 \\
    L-8B & \dkl / Match.   & 0.58 & 100  & 96.6 & 0.831 & 0.697 & 44 \\
    \midrule
    \multicolumn{8}{l}{\textit{Mistral-7B - best closed-book}} \\
    M-7B & RAFT / Inst.$^{\ddagger}$ & \textbf{0.73} & 98.8 & 98.3 & \textbf{0.826} & \textbf{0.708} & \textbf{38} \\
    M-7B & Normal / Match. & 0.55 & 100  & 96.6 & 0.816 & 0.580 & 33 \\
    \midrule
    \multicolumn{8}{l}{\textit{Qwen3-4B - best closed-book (Matched \slora)}} \\
    Q-4B & \dkl / Match.    & 0.72 & 100 & 96.6 & \textbf{0.813} & 0.629 & 36 \\
    Q-4B & Unified / Match. & 0.63 & 100 & 98.3 & 0.810 & 0.630 & 37 \\
    Q-4B & Normal / Match.  & 0.65 & 100 & 100  & 0.796 & \textbf{0.645} & \textbf{41} \\
    \bottomrule
  \end{tabular}
\end{table}

\textbf{LLaMA.} The chain is clearest for LLaMA. LLaMA-Unified, which
demonstrated the strongest world-model induction evidence
($\tauorderu = 0.81$, near-equal $+0.26$ gains on seen and unseen, inverted
gap), achieves the highest BT \tauorder ($0.848$) under Instruct mounting. The
compositional generalisation capacity that produced an inverted \orderspatial
gap - the model ordering novel scenes at least as well as seen ones -
transfers directly into compositionally correct BT structures. LLaMA-Overfit
ties at $0.848$ under Matched mounting despite $\tauorderu = 0.54$,
illustrating the deep-parametric-encoding pathway of
Section~\ref{sec:slora_results}: a different mechanism, equally effective in the
final metric. LLaMA-SSR-CLM / Matched ($0.844$, $\tauorderu = 0.74$) further
reinforces that strong \orderspatial concordance is recoverable under the right
mounting strategy even for a variant that collapses under Instruct.

\textbf{Mistral.} The chain closes cleanly for Mistral through a single
variant: RAFT. Mistral-RAFT was the only Mistral variant to show genuine
world-model induction evidence, an \emph{inverted} seen/unseen gap
($\tauorderu = 0.73 > 0.71$ seen, $\Delta^u = +0.21 > \Delta^s = +0.18$). It is
also the best Mistral BT result overall ($\tauorder = 0.826$, Instruct
mounting). The correspondence is precise: the one
Mistral variant whose benchmark profile reflects abstract-layer internalisation
rather than surface memorisation is also the one that best translates
internalised knowledge into executable plans.

\textbf{Qwen.} Qwen's BT results are uniformly high across Matched variants
($\tauorder \in [0.796, 0.813]$, a range of only $0.017$), making per-variant
ranking less informative than for LLaMA and Mistral. Qwen-\dkl leads narrowly
($0.813$, $\tauorderu = 0.72$), with Qwen-Unified close behind ($0.810$). The
flatness of the Qwen band reflects the architecture's built-in chain-of-thought
reasoning, which provides a strong structural prior for BT generation
regardless of which KI variant is mounted, partially decoupling BT quality from
the fine-grained differences in knowledge internalisation that \orderspatial
distinguishes.

\textbf{Broader observation.} GPT-4.1 vanilla achieves $\tauorder = 0.480$ on
BT generation and $0.44$ on \orderspatial, just below the $0.500$ random
baseline. The gap between GPT-4.1 with structured retrieval and the best
Domain-BT-LMs is not merely quantitative; it reflects a qualitative difference
in how priority knowledge is held. GPT-4.1 retrieves and applies rules at
inference; Domain-BT-LMs have internalised the interaction grammar
parametrically and express it directly through plan structure. For
domain-specific safety-critical manipulation, a properly trained open-source
SLM substantially outperforms a frontier proprietary model with retrieval,
while avoiding cloud latency, API cost, and data-privacy
exposure~\citep{nasrat2025rdmm}.

\section{Discussion}
\label{sec:discussion}

\paragraph{\order as evaluation infrastructure}
Without \order, a researcher applying a continual-learning pipeline to a
proprietary corpus cannot determine whether improvements represent genuine
knowledge acquisition or pre-training pattern retrieval. \order converts this
uncertain process into a validated one: candidate pipelines are evaluated on
\orderbench and \orderspatial, identifying those with small post-KI seen/unseen
gaps and large absolute improvements on both splits, and the shortlist is then
confirmed by post-merge \svr validation and downstream BT concordance. The
analogy to flight simulators is apt: the synthetic environment does not
replicate real domain semantics, but replicates the \emph{structural
properties} that make the real problem hard.

\paragraph{The retrieval--CPT division of labour}
Retrieval serves factual recall well, as non-spatial MCQ saturation confirms,
while compositional application of multiple rules to novel configurations
requires parametric internalisation via
CPT~\citep{yang2026fine,mallen2023not}. The same dissociation re-appears one
level down, in executable plan generation: factual priority resolution is
well-served by a condensed retrieval call, whereas compositional plan
expression - how to express an ordering as a syntactically correct,
semantically coherent BT - requires the parametric skill internalised by
\slora. Production KHTL systems should combine both: CPT for the compositional
reasoning procedure and the generation skill, retrieval for domain facts that
evolve.

\paragraph{The systematicity debate}
Unlike SCAN~\citep{lake2018scan}, where modern LLMs succeed largely because
command rules appear in pre-training corpora~\citep{drozdov2023compositional},
\order eliminates this confound entirely. LLaMA-Unified's $\tauorderu = 0.81$
with equal absolute improvement on both splits ($+0.26$ each) provides evidence
for abstract-layer internalisation over a genuinely novel axiomatic
physics~\citep{fodor1988connectionism}, and the same variant produces the
strongest executable plans.

\paragraph{Skill--knowledge disentanglement as a design principle}
The Instruct mounting strategy - training \slora from the instruct parent and
mounting onto any KI-variant - produces competitive and in some cases
superior results compared to Matched mounting. This validates the
disentanglement principle of Eq.~\ref{eq:domainbt}: the skill of BT generation
is largely separable from the domain knowledge of \orderworld, and the two can
be composed at inference via task arithmetic without joint training. The
LLaMA-Unified result illustrates this most sharply: its highest BT performance
($0.848$) occurs precisely under Instruct mounting (clean disentanglement),
while Matched mounting drops to $0.601$. The practical consequence is that a
single \slora adapter trained from the instruct parent can be shared across all
KI-variants of a given model family, substantially reducing adaptation cost for
KHTL deployments. A complementary observation reinforces the principle from
another angle: LLaMA-RAFT, whose \tauorderu of $0.60$ is only moderately above
the random baseline and whose MCQ aggregate barely clears the zero-shot floor,
achieves BT $\tauorder = 0.833$ under Instruct mounting. Compatibility of the
CPT objective's \emph{conditioning geometry} with the downstream skill is
therefore a meaningful predictor even when absolute knowledge retention is
modest. This observation is architecture-specific: it does not hold for Qwen.

\paragraph{Constrained start vs.\ full grammar decoding}
That full CFG-guided decoding degrades BT quality while prefix forcing helps
has a principled explanation. \slora training implicitly learns a joint
distribution over valid XML structures; CFG-guided decoding imposes a
\emph{marginal} constraint at each step that may be inconsistent with the
\slora-trained conditional distribution, forcing probability mass onto
grammatically valid but contextually inconsistent tokens. Prefix forcing
commits the initial state (the opening tag) without further restricting
subsequent generation, allowing the fine-tuned distribution to govern the
remaining output. The primary failure mode in BT generation is therefore the
initial output-mode selection (prose vs.\ XML), not mid-sequence structural
errors - a finding with practical implications for any structured generation
task where fine-tuned models must switch output modalities.

\section{Limitations}
\label{sec:limitations}
Gold labels are generated using GPT-4.1, which is simultaneously an evaluated
frontier baseline; complex scenes may inherit subtle oracle biases, and a fully
symbolic generation procedure would remove this dependency. Structural validity
and uniqueness are formally guaranteed (Appendix~\ref{app:order_theory}). Appendix~\ref{app:general_benchmarks} measures general-capability
retention but not real-world physical or spatial judgement specifically.
\order is a synthetic environment
with clean, fully specified physics, so strong performance is necessary but not
sufficient evidence of real-world transfer; the HITL framework is evaluated on
the simulated \texttt{iiwa7}, with Doosan A0509S hardware deployment planned.
The \slora training set (424 pairs) is relatively small, though scaling via
additional scene generation is straightforward within the \orderworld pipeline.
Finally, the precise conditions under which knowledge ingestion provides
decisive gains over pure skill fine-tuning remain open, and are likely to
depend on the semantic difficulty of the priority reasoning required rather
than on BT format compliance alone.

\section{Conclusion}
\label{sec:conclusion}

We introduced \order, a fictitious-world framework for contamination-free
evaluation of knowledge ingestion in domain-adaptive embodied AI, comprising
\orderworld (a 342k-token synthetic corpus), \orderbench (500 MCQ), and
\orderspatial (a 253-scene seen/unseen priority-ordering task with formally
guaranteed unique gold labels). Retrieval saturates factual recall
($\geq\!90\%$) but leaves compositional spatial ordering substantially harder.
GPT-4.1 without domain adaptation scores below random on \orderspatial
($\tauorder = 0.441$), confirming genuinely novel physics. After CPT,
best-adapted SLMs show $+0.18$--$+0.26$ \tauorder on both seen and unseen
splits with gaps $\leq\!0.05$, providing strong evidence for abstract-layer
internalisation over memorisation.

We then carried the framework through to a robot. Domain knowledge and the
skill of BT generation are separable capabilities: KI-models with strong
\orderbench and \orderspatial performance consistently fail to generate
executable BTs without task-specific \slora fine-tuning. Injecting a full
assembled ontology context into a fine-tuned SLM causes context hijacking,
collapsing \svr to $0\%$; a two-stage pipeline with a condensed
priority-reasoning call, optionally with constrained-start prefix forcing,
recovers functional outputs, while full grammar-constrained decoding at every
step degrades quality. A HITL framework on the simulated \texttt{iiwa7} arm
demonstrates that single-correction re-prompting reliably produces compilable,
correctly ordered trees. Across all configurations, the best closed-book
Domain-BT-LMs substantially outperform GPT-4.1 with structured retrieval
($\tauorder = 0.848$ vs.\ $0.606$).

The connection between the benchmark and the robot is tight and non-circular:
the variants that demonstrated the strongest world-model induction evidence
(LLaMA-Unified, with near-equal seen/unseen gains and an inverted post-KI gap;
Mistral-RAFT, the one variant in its family with an inverted gap) are precisely
those that achieve the strongest BT generation results, and \orderspatial
compositional concordance is confirmed as a stronger predictor of downstream
plan quality than MCQ aggregate across three model families and dozens of
experimental configurations. \order is released as pipeline-agnostic community
infrastructure for KHTL robotic system development.

\section*{Acknowledgements}
The authors sincerely thank the HPCE facility at IIT Madras (AQUA cluster),
whose computing resources were essential for all continual pre-training and
fine-tuning experiments.

\bibliographystyle{IEEEtran}
\bibliography{BIB}

\appendices

\section{Retrieval Pipeline Specifications}
\label{app:rag}

The retrieval baselines reported in Table~\ref{tab:baseline_rag} and
Section~\ref{sec:ragresults} are not naive retrieve-and-read systems. This
appendix documents their full architectures to establish that the persistent
difficulty on compositional spatial priority ordering is a fundamental
limitation of retrieval-augmented inference, not an artefact of implementation
quality.

\subsection{MCQ benchmark retrieval: global hybrid retrieval}
\label{app:rag_mcq}

The pipeline used for Table~\ref{tab:baseline_rag} is a global hybrid retriever
searching the entire \orderworld corpus (342{,}069 tokens) without section
filtering.

\textbf{Stage 1: parallel candidate retrieval.}
\begin{itemize}[leftmargin=1.4em,itemsep=1pt,topsep=2pt]
  \item \textbf{Sparse (BM25 Okapi):} BM25-weighted scoring; top 50 candidates
    retained.
  \item \textbf{Dense (all-MiniLM-L6-v2):} 384-d sentence embedding via
    dot-product similarity; top 50 candidates retained.
\end{itemize}

\textbf{Stage 2: reciprocal rank fusion.}
\begin{equation}
  \mathrm{RRF}(d)
    = \sum_{i \,\in\, \{\mathrm{BM25,\,dense}\}}
      \frac{1}{k + \mathrm{rank}_i(d)}, \qquad k = 60.
  \label{eq:rrf}
\end{equation}
The top-5 passages by RRF score are concatenated and prepended to the MCQ
question; an identical template is applied to all models.

\subsection{BT generation retrieval: structured ontology-guided index}
\label{app:rag_bt}

The BT retrieval pipeline maintains separate exact-match and vector indices for
each knowledge type.

\textbf{Multi-type knowledge decomposition.}
\begin{itemize}[leftmargin=1.4em,itemsep=1pt,topsep=2pt]
  \item \textbf{Exact-match stores} ($O(1)$ lookup): intrinsic semantics
    indexed by primitive token; pairwise laws indexed by unordered primitive
    pair; counterfactual constraints indexed by primitive pair.
  \item \textbf{FAISS vector index} (all-MiniLM-L6-v2): higher-order regimes,
    dominance axioms, and QnA entries.
\end{itemize}

\textbf{Weighted context assembly.} Retrieved chunks are ranked within a
4{,}000-token budget according to Table~\ref{tab:rag_weights}. For the
two-stage pipelines of Section~\ref{sec:four_pipelines}, the assembled context
is further condensed to $\leq$800 tokens (top-5 counterfactual, axiom, and
inter-object pairwise chunks) before the Stage~1 priority-reasoning call.

\begin{table}[ht]
  \centering
  \caption{Knowledge-type weighting used for context assembly.}
  \label{tab:rag_weights}
  \renewcommand{\arraystretch}{1.12}
  \footnotesize
  \begin{tabular}{lcc}
    \toprule
    \textbf{Knowledge Type} & \textbf{Weight} & \textbf{Order} \\
    \midrule
    Counterfactual constraints      & 1.00 & 1st \\
    Pairwise inter-object (contact) & 0.90 & 2nd \\
    Dominance axioms                & 0.85 & 3rd \\
    Higher-order emergent regimes   & 0.80 & 4th \\
    Pairwise intra-object           & 0.70 & 5th \\
    Intrinsic primitive semantics   & 0.55 & 6th \\
    QnA cross-format entries        & 0.40 & 7th \\
    \bottomrule
  \end{tabular}
\end{table}

\textbf{Few-shot conditioning.} Two training-set examples are retrieved per
test scene by structural similarity. Spatial configuration context is
deliberately excluded (\emph{no-spatial} mode), as this ablation consistently
outperformed pipelines that included spatial scene descriptions.

\subsection{On the sophistication of the retrieval baselines}
The claim is not that a poorly engineered retrieval system fails on
\orderspatial, but that retrieval-augmented inference has a structural
limitation for compositional planning: correctly retrieving relevant rules does
not perform their joint application to a novel scene. Composing multiple rules
simultaneously requires the interaction grammar to be available as a parametric
procedure; CPT internalises this grammar into model weights.

\section{Order-Theoretic Grounding}
\label{app:order_theory}

\begin{definition}[Precedence Poset]
A precedence poset $\poset = (\objects, \prec)$ with objects
$\objects = \{o_1,\ldots,o_n\}$ where $o_i \prec o_j$ means $o_i$ must be
manipulated before $o_j$; $\prec$ is irreflexive, transitive, and asymmetric.
\end{definition}

\begin{definition}[Linear Extension]
A bijection $\pi:\{1,\ldots,n\}\to\objects$ such that
$o_i\prec o_j \Rightarrow \pi^{-1}(o_i)<\pi^{-1}(o_j)$. When $\prec$ is a
strict total order, exactly one linear extension exists.
\end{definition}

The ontology induces a criticality score
$\score(o_i,\config) = \phi(\mathbf{a}(o_i))
+ \sum_{j\neq i}\psi(\mathbf{a}(o_i), \mathbf{a}(o_j), r(o_i,o_j))$,
where $\phi$ encodes intrinsic dominance and $\psi$ encodes pairwise and
higher-order interaction contributions. The priority relation is
$o_i\prec o_j\Leftrightarrow\score(o_i,\config)>\score(o_j,\config)$.

\textbf{Pillar 1 (Closed-World Completeness).}
Under the Closed World Assumption~\citep{reiter1978closed} and Unique Name
Assumption~\citep{brachman2004kr}, the knowledge base
$K=\mathcal{T}\cup\mathcal{X}$ satisfies
$\forall o_i\neq o_j: K\models(o_i\prec o_j)\vee K\models(o_j\prec o_i)$.

\textbf{Pillar 2 (Strict Total Order by Construction).}
Injectivity of $\score$ is guaranteed by: no two objects in any benchmark scene
share an identical four-axis signature, and a deterministic lexicographic
tiebreaker~\citep{fishburn1974lexicographic} (Colour $\succ_{\mathrm{lex}}$
Shape $\succ_{\mathrm{lex}}$ Size $\succ_{\mathrm{lex}}$ Texture) is encoded in
the ontology contract.

\begin{proposition}
\label{prop:total_order}
Under the above conditions, $\prec$ is a strict total order on $\objects$ for
every \orderspatial scene, so $|\mathcal{L}(\poset)|=1$ and the gold sequence
$\pi^*$ is unique.
\end{proposition}
\begin{proof}
\textbf{Irreflexivity.} $o_i\prec o_i$ requires
$\score(o_i,\config)>\score(o_i,\config)$, a contradiction.
\textbf{Transitivity.} Follows from transitivity of $>$ on $\mathbb{R}$.
\textbf{Totality.} For distinct $o_i,o_j$, either the scores differ (direct
comparability) or the deterministic lexicographic tiebreaker resolves the tie;
it is a strict total order on $\mathcal{A}$ guaranteed to resolve every tie
since no two objects share an identical four-axis signature.
\textbf{Uniqueness.} A poset admits a unique linear extension iff its partial
order is a strict total order~\citep{davey2002lattices}; since $\prec$ is
total, $|\mathcal{L}(\poset)|=1$.
\end{proof}

\textbf{Pillar 3 (Oracle Methodology).}
Gold orders are generated by GPT-4.1 via prompt-forced, citation-grounded
reasoning over all activated context blocks, cross-validated against the axiom
set with human review for unresolved activations (full specifications in
Appendix~\ref{app:gen:oracle}).

\section{PPL Profiles and the PPL--Spatial Dissociation}
\label{app:ppl}

Table~\ref{tab:ppl} reports perplexity and training loss for all \klora
variants on a held-out split of $\corpus$. Lower PPL broadly correlates with
higher $S_{\mathrm{agg}}$ as a rough monitoring signal, but two critical
dissociations establish the limit of that signal.

SSR-CLM achieves the lowest PPL across all three families (Mistral: 1.52,
LLaMA: 2.52, Qwen: 4.62) yet does not achieve correspondingly strong spatial
reasoning: Mistral-SSR-CLM achieves the weakest spatial MCQ ($43.1\%$) among
all non-collapsed Mistral variants. The loss-masked SSR objective drives format
conformity on spatial entries rather than rule internalisation under the
all-token CLM signal.

Mistral-Unified collapses catastrophically ($S_{\mathrm{agg}} = 1.8\%$) while
retaining PPL of 2.15, with neither perplexity nor loss curves giving any
warning before post-merge evaluation. The most likely cause is directional
interference in LoRA merging when combining CLM, structured SFT, and DPO
objectives simultaneously. These two cases jointly establish that post-merge
evaluation on \orderbench is a necessary validation step.

\begin{table}[ht]
  \centering
  \caption{PPL and training loss for all \klora variants. SSR-CLM achieves the
    lowest PPL but dissociates from spatial performance. Mistral-Unified
    collapses silently.}
  \label{tab:ppl}
  \setlength{\tabcolsep}{4pt}
  \renewcommand{\arraystretch}{1.1}
  \footnotesize
  \begin{tabular}{l cc cc cc}
    \toprule
    & \multicolumn{2}{c}{\textbf{M-7B}}
    & \multicolumn{2}{c}{\textbf{L-8B}}
    & \multicolumn{2}{c}{\textbf{Q-4B}} \\
    \cmidrule(lr){2-3}\cmidrule(lr){4-5}\cmidrule(lr){6-7}
    \textbf{Variant} & PPL & Loss & PPL & Loss & PPL & Loss \\
    \midrule
    Normal    & 3.82 & 0.230 & 7.22 & 0.820 & 8.89 & 1.110 \\
    Overfit   & 3.10 & 0.021 & 6.04 & 0.016 & 6.79 & 0.059 \\
    \dkl      & 2.07 & 0.027 & 4.72 & 0.015 & 6.06 & 0.085 \\
    RAFT      & 3.44 & 0.823 & 5.27 & 1.660 & 6.53 & 2.650 \\
    SSR-CLM   & \textbf{1.52} & 0.003
              & \textbf{2.52} & 0.002
              & \textbf{4.62} & 0.001 \\
    Unified   & 2.15 & 0.003 & 4.73 & 0.008 & 7.19 & 0.017 \\
    Combined  & 2.79 & NA    & 3.07 & NA    & 6.99 & NA    \\
    \bottomrule
  \end{tabular}
\end{table}

\section{\orderspatial Scene Complexity Distribution}
\label{app:complexity}

\begin{table*}[ht]
\centering
\footnotesize
\setlength{\tabcolsep}{6pt}
\renewcommand{\arraystretch}{1.12}
\caption{Complexity distribution across \orderspatial splits. Seen and unseen
  are statistically comparable across all axes (Welch $t$-tests: all $|t|<2$,
  $p>0.05$). The mild four-object skew in unseen ($30.6\%$ vs.\ $18.6\%$)
  slightly disfavours unseen performance.}
\label{tab:complexity_distribution}
\begin{tabular}{lcccc c}
\toprule
& \multicolumn{2}{c}{\textbf{Seen} ($n=70$)}
& \multicolumn{2}{c}{\textbf{Unseen} ($n=183$)} & \\
\cmidrule(lr){2-3}\cmidrule(lr){4-5}
\textbf{Metric} & \textbf{Mean} & \textbf{Std}
                & \textbf{Mean} & \textbf{Std} & \textbf{$t$} \\
\midrule
\multicolumn{6}{l}{\textit{Structural}} \\
Num.\ objects        & 3.29 & 0.97 & 3.53 & 0.97 & $-1.76$ \\
Unique primitives    & 9.53 & 2.06 & 9.63 & 2.47 & $-0.34$ \\
Spatial relations    & 2.29 & 0.97 & 2.45 & 0.98 & $-1.19$ \\
\midrule
\multicolumn{6}{l}{\textit{Physics (activation counts)}} \\
Pairwise rules       & 21.33 & 11.21 & 23.13 & 10.36 & $-1.17$ \\
Higher-order regimes & 1.33  & 1.12  & 1.30  & 1.29  & $\phantom{-}0.18$ \\
Axioms               & 0.66  & 1.75  & 0.96  & 2.53  & $-1.07$ \\
Counterfactuals      & 8.47  & 4.42  & 8.90  & 4.71  & $-0.68$ \\
\midrule
\multicolumn{6}{l}{\textit{Composite}} \\
\texttt{complexity\_score}$^\dagger$
                     & 68.99 & 32.36 & 73.45 & 30.74 & $-0.99$ \\
\midrule
\multicolumn{6}{l}{\textit{Higher-order regime presence}} \\
Scenes with $\geq\!1$ active regime
& \multicolumn{2}{c}{53/70 (75.7\%)}
& \multicolumn{2}{c}{127/183 (69.4\%)}
& $\phantom{-}0.18$ \\
\bottomrule
\end{tabular}
\vspace{3pt}
\begin{minipage}{0.8\textwidth}
\footnotesize
$^\dagger$\,$\texttt{complexity\_score} = n_\text{pw} + 3\,n_\text{ho}
+ 2\,n_\text{ax} + 5\,n_\text{cf}$.
All $t$-statistics are two-sided Welch $t$-tests; all $p > 0.05$.
\end{minipage}
\end{table*}

\subsection{Statistical validity of \orderspatial results}

\textbf{Finite-sample uncertainty bounds.} By Hoeffding's
inequality~\citep{hoeffding1963}, the distribution-free 95\% confidence
half-width is $\varepsilon^* = \sqrt{\ln(40)/2n}$. For $n_s = 70$ and
$n_u = 183$: $\varepsilon^*_s = 0.162$, $\varepsilon^*_u = 0.100$. The
LLaMA-Unified improvement $\Delta\tau = {+}0.26$ exceeds the conservative
Hoeffding half-width by $2.6\times$.

\textbf{Improvement significance against random baseline.} Under
$H_0\colon \tau \le 0.500$ with $\sigma_{\max} = 0.5$, the one-sided
$z$-statistic is $z = (\hat\tau - 0.500)/(\sigma_{\max}/\!\sqrt{n})$.
Table~\ref{tab:zsig} reports results; all six post-KI results are significant
at $p < 0.05$. GPT-4.1 unseen: $z = -2.30$, $p = 0.011^{*}$, confirming that
the frontier model's ordering is significantly anti-correlated with the gold
sequence.

\begin{table}[ht]
\centering
\caption{One-sided $z$-test results. $^{*}p<0.05$, $^{**}p<0.01$,
  $^{***}p<0.001$.}
\label{tab:zsig}
\footnotesize
\renewcommand{\arraystretch}{1.08}
\begin{tabular}{llcccc}
\toprule
Family & Variant & $\hat\tau_s$ & $z_s$ & $\hat\tau_u$ & $z_u$ \\
\midrule
M-7B    & RAFT    & 0.710 & $+3.51^{***}$ & 0.730 & $+6.20^{***}$ \\
L-8B    & Unified & 0.770 & $+4.52^{***}$ & 0.810 & $+8.39^{***}$ \\
Q-4B    & Combined& 0.790 & $+4.85^{***}$ & 0.740 & $+6.49^{***}$ \\
\midrule
GPT-4.1 & Vanilla & 0.441 & $-0.99$ & 0.415 & $-2.30^{*}$ \\
\bottomrule
\end{tabular}
\end{table}

\section{CPT Algorithm Specifications}
\label{app:algorithms}

All variants share three design invariants: CPT is applied to the base model
$\theta_B$; the instruct model's embedding matrix is transplanted into
$\theta_B$ before training; and the trained adapter $\tau$ is merged via task
arithmetic, $\theta_{KI} = \theta_I + \lambda\,\tau$ with $\lambda = 0.5$.
Shared LoRA hyperparameters: $r=64$, $\alpha=128$, dropout $=0$, all linear
layers, lr $5\!\times\!10^{-5}$, 10 epochs, effective batch size 4.

\begin{algorithm}[ht]
\caption{\dkl Knowledge Ingestion}
\label{alg:dkl}
\begin{algorithmic}[1]
\Require Base model $\theta_B$; instruct model $\theta_I$; corpus $\corpus$;
  rank $r$; scale $\alpha$; lr $\eta$; epochs $E$; max sequence length $L$
\Ensure  $\theta_{KI}$
\State \textbf{// Phase 1 -- \dkl embedding replacement}
\State Load $\theta_I$ onto CPU in bfloat16
\If{$|\mathrm{vocab}(\theta_I)| \neq |\mathrm{vocab}(\theta_B)|$}
  \State Resize $\theta_B$.\texttt{embed\_tokens} and $\theta_B$.\texttt{lm\_head}
\EndIf
\State $\theta_B$.\texttt{embed\_tokens.weight} $\leftarrow$
  $\theta_I$.\texttt{embed\_tokens.weight.clone()}
\State $\theta_B$.\texttt{lm\_head.weight} $\leftarrow$
  $\theta_I$.\texttt{lm\_head.weight.clone()}
\State Freeze both layers; free $\theta_I$; flush CUDA
\State \textbf{// Phase 2 -- LoRA init}
\State Attach LoRA adapter $\tau$ to $\theta_B$
\State \textbf{// Phase 3 -- Dataset: sequence packing}
\State Tokenise $\corpus$; partition into chunks of length $L$
\State Labels $\leftarrow$ input IDs (all-token CLM loss)
\State \textbf{// Phase 4 -- Training}
\For{epoch $e = 1, \ldots, E$}
  \For{each packed chunk $\mathbf{x}$}
    \State Update $\tau$ via AdamW on $\mathcal{L}_{\mathrm{CLM}}$
  \EndFor
\EndFor
\State $\theta_{KI} \leftarrow \theta_I + 0.5\cdot\tau$
\State \Return $\theta_{KI}$
\end{algorithmic}
\end{algorithm}

\begin{algorithm}[ht]
\caption{RAFT Spatial Knowledge Ingestion}
\label{alg:raft}
\begin{algorithmic}[1]
\Require $\theta_B$; $\theta_I$; $\corpus_{sp}$; retrieval index $\mathcal{R}$;
  top-$k$; hard-negative fraction $p_{\mathrm{neg}}$
\Ensure  $\theta_{KI}$
\State Apply Phases 1--2 of Algorithm~\ref{alg:dkl}
\For{each $s \in \corpus_{sp}$}
  \State $P^{+} \leftarrow \mathcal{R}.\mathrm{retrieve}(s,k)$
  \If{Bernoulli$(p_{\mathrm{neg}})$}
    \State $s' \sim \corpus_{sp} \setminus \{s\}$;
      $P^{-} \leftarrow \mathcal{R}.\mathrm{retrieve}(s', k)$
    \State Add $([\mathrm{scene}(s)\,;\,P^{-}],\;\mathrm{analysis}(s))$
  \Else
    \State Add $([\mathrm{scene}(s)\,;\,P^{+}],\;\mathrm{analysis}(s))$
  \EndIf
\EndFor
\State Format as instruction-following; mask loss to assistant turn
\State Apply the training and merge phases of Algorithm~\ref{alg:dkl}
\State \Return $\theta_{KI}$
\end{algorithmic}
\end{algorithm}

\begin{algorithm}[ht]
\caption{SSR-CLM Knowledge Ingestion}
\label{alg:ssrclm}
\begin{algorithmic}[1]
\Require $\theta_B$; $\theta_I$; $\corpus_{ns}$; $\corpus_{sp}$
\Ensure  $\theta_{KI}$
\State Apply Phases 1--2 of Algorithm~\ref{alg:dkl}
\State \textbf{// Phase 3A -- Packed CLM on $\corpus_{ns}$}
\State Labels $\leftarrow$ input IDs (all-token CLM loss)
\State \textbf{// Phase 3B -- SSR on $\corpus_{sp}$}
\For{each $s \in \corpus_{sp}$}
  \State Generate four masked examples:
    (i)~full eight-step analysis;
    (ii)~interface identification;
    (iii)~dominance conflict resolution;
    (iv)~priority ordering output
  \State Mask labels on system$+$user tokens; loss on assistant tokens only
\EndFor
\State Repeat CLM chunks to equalise per-epoch coverage; interleave and shuffle
\State Train $\tau$; $\theta_{KI} \leftarrow \theta_I + 0.5\cdot\tau$
\State \Return $\theta_{KI}$
\end{algorithmic}
\end{algorithm}

\begin{algorithm}[ht]
\caption{Unified Knowledge Ingestion}
\label{alg:unified}
\begin{algorithmic}[1]
\Require $\theta_B$; $\theta_I$; $\corpus_{ns}$; $\corpus_{sp}$;
  $\beta\!=\!0.1$; $p_{\mathrm{dpo}}\!=\!0.3$
\Ensure  $\theta_{KI}$
\State Apply Phases 1--2 of Algorithm~\ref{alg:dkl}
\State Build packed CLM from $\corpus_{ns}$
\For{each $s \in \corpus_{sp}$}
  \State Add distractor-robustness, counterfactual-delta, context-faithful,
    and explanation-mapping examples
  \If{Bernoulli$(p_{\mathrm{dpo}})$}
    \State Add DPO pair: chosen $=$ oracle-faithful, rejected $=$
      parametric-bias
  \EndIf
\EndFor
\State Balance and shuffle dataset
\For{each mini-batch $\mathcal{B}$}
  \State $\mathcal{L}_{\mathrm{CE}}$ on non-DPO examples
  \If{DPO examples present}
    \State $\mathcal{L}_{\mathrm{DPO}}$ (reference-free)
  \EndIf
  \State Update $\tau$ on mean loss
\EndFor
\State $\theta_{KI} \leftarrow \theta_I + 0.5\cdot\tau$
\State \Return $\theta_{KI}$
\end{algorithmic}
\end{algorithm}

\begin{algorithm}[ht]
\caption{Combined: SSR-CLM adapter $+$ RAFT adapter}
\label{alg:combined}
\begin{algorithmic}[1]
\Require $\theta_I$; $\tau_{\mathrm{SSR}}$; $\tau_{\mathrm{RAFT}}$
\Ensure  $\theta_{KI}$
\State $\theta_{KI} \leftarrow \theta_I
  + 0.5\cdot\tau_{\mathrm{SSR}} + 0.5\cdot\tau_{\mathrm{RAFT}}$
\State \Return $\theta_{KI}$
\end{algorithmic}
\end{algorithm}

\section{Benchmark Generation Pipeline}
\label{app:generation}

\orderbench is produced by a fully deterministic, multi-stage GPT-4.1 oracle
pipeline enforcing closed-world consistency, Bloom-level cognitive depth, and
answer-option discriminability.

\subsection{\orderworld corpus generation}
\label{app:gen:corpus}

\orderworld is constructed as a six-layer knowledge pyramid; every downstream
generation call receives the complete text of all upstream layers as injected
context. All calls use GPT-4.1; temperatures are low ($T\!\in\![0.2,0.4]$) for
structurally constrained layers, raised to $T\!=\!0.6$ for intrinsic-semantics
generation to produce distinct physical personalities across the primitives.

\subsubsection{Ontological Contract (Layer 0.0)}
\texttt{generate\_section\_00\_contract.py} generates the ontological contract,
defining the visual grammar $\mathcal{G}$ and encoding four inviolable
invariants: the Closed World Assumption~\citep{reiter1978closed}; the Unique
Name Assumption~\citep{brachman2004kr}; prohibition on semantic extension by
any downstream module; and the lexicographic tiebreaker guaranteeing the strict
total order of Proposition~\ref{prop:total_order}.

\subsubsection{Primitive Vocabulary (Layer 0.1)}
\texttt{generate\_section\_01\_primitives.py} produces a formal,
interaction-free symbol catalogue across five sub-sections (shape, colour,
size, texture, spatial relation). A strict closed-set rule is enforced:

\begin{promptbox}
Rules:\\
- You must define every primitive listed above.\\
- You must NOT introduce any primitive, synonym, subtype,\\
\phantom{- }or alternative label beyond these sets.\\
- You must NOT remove or merge primitives.\\
- You must NOT rename primitives.\\
- Each primitive must be treated as an atomic symbol.\\
\\
Global rules:\\
- No interaction semantics\\
- No examples involving actions or outcomes\\
- No reference to risk, safety, or hazards\\
- Use formal, contract-style language
\end{promptbox}

\subsubsection{Intrinsic Semantics (Layer 0.2)}
Per-primitive, per-axis generation with hard-coded semantic axis assignments:
colours (4 axes: Flammability, Chemical Reactivity, Biodegradability, Hazard
Class); shapes (4 axes: Load Distribution, Mechanical Stability, Containment,
Fracture Initiation); sizes (3 axes: Surface Dominance, Inertial Effect,
Threshold Sensitivity); textures (4 axes: Permeability, Reaction Amplification,
Wear Visibility, Contamination Retention); 77 assignments in total. The system
prompt enforces a declarative physical-assertion style:

\begin{promptbox}
You are writing the FOUNDATIONAL PHYSICS of a\\
SELF-CONTAINED WORLD.\\
\\
This is NOT a description of an ontology.\\
This IS the ontology.\\
\\
LANGUAGE RULES:\\
- Speak as if these laws are universally true\\
- Use declarative physical assertions only\\
- No epistemic framing (``defined as'', ``treated as'')\\
- No meta-commentary\\
\\
FORBIDDEN:\\
- Optics or perception; Real-world chemistry\\
- Probabilistic language; Interactions or agents\\
- Meta phrases (``ontology'', ``framework'', ``mapping'')\\
\\
This text will be read as PHYSICS, not documentation.
\end{promptbox}

\subsubsection{Pairwise Interaction Laws (Layer 0.3)}
86 curated pairs spanning seven interaction categories; each call receives the
full Layer~0.2 intrinsic physics of both primitives. Interactions are derived
from field coupling only:

\begin{promptbox}
CRITICAL CONSTRAINTS:\\
- The physics from Section~0.2 is IMMUTABLE\\
- You are deriving INTERACTION EFFECTS from established\\
\phantom{- }intrinsic fields\\
- Interactions arise from FIELD COUPLING, not new invention\\
\\
FORBIDDEN:\\
- New fundamental properties not in Section~0.2\\
- Probabilistic or uncertain language\\
- Scenarios, agents, or narratives
\end{promptbox}

\subsubsection{Higher-Order Emergent Regimes (Layer 0.4)}
37 curated triadic combinations spanning six regime families. Each call
receives the intrinsic physics of all three primitives \emph{and} all three
pairwise interaction laws. The oracle must derive a regime not reducible to
pairwise sums:

\begin{promptbox}
CRITICAL REQUIREMENTS:\\
- The result MUST NOT be reducible to a sum of pairwise laws\\
- At least one pairwise interaction MUST become secondary,\\
\phantom{- }suppressed, or gated\\
- All effects MUST arise from coupling, saturation,\\
\phantom{- }resonance, or masking of existing fields
\end{promptbox}

\subsubsection{Axioms, Counterfactuals, QnA and Paraphrases
  (Layers 1, 2, 3, 5)}
80 axiom themes spanning ten axiom types; three Inference meta-axioms close the
system. Counterfactuals sample from the latent interaction space (pairs and
triads not covered by Layers 0.3--0.4); each call receives only Layer~0.2
intrinsic physics. Layers~3 and~5 produce surface-varied re-expressions
following the multi-format CPT principle of~\citep{allen2023physlm}; no new
semantic content is introduced.

\subsubsection{Spatial Configurations (Layer 4)}
Before the oracle call, a deterministic five-step semantic activation graph
identifies relevant corpus entries: (i)~intra-object pairs per object's
attribute axes; (ii)~inter-object pairs at active contact interfaces;
(iii)~higher-order regimes for applicable triads; (iv)~axioms whose source
contexts are fully satisfied; (v)~counterfactuals whose canonical pair appears
in active sets. The oracle (GPT-4.1, $T\!=\!0.4$) receives the scene
description and all five activated context blocks, then outputs a structured
eight-step analysis culminating in an explicit \textsc{Manipulation Priority
Guidance} section.

\subsection{MCQ generation prompts}
\label{app:gen:mcq}

Shared system prompt for all MCQ generation calls:

\begin{promptbox}
You are generating evaluation questions for a closed-world\\
physics benchmark.\\
The physics described in the context is the ONLY physics\\
that exists.\\
You must not import real-world chemistry, common-sense\\
ordering priors, or any knowledge not present in the\\
provided corpus text.\\
Every question must have exactly one correct answer\\
directly recoverable from the provided context.\\
Every distractor must be incorrect for a specific,\\
identifiable reason grounded in the corpus.\\
Probabilistic language, hedges, and open-world inferences\\
are forbidden.
\end{promptbox}

Per-category distractor design:
\textbf{Intrinsic} ($N\!=\!10$, $w_c\!=\!1.0$): three plausible alternative
axis values.
\textbf{QnA} (30, 1.0): wrong answers from different QnA entries.
\textbf{Axioms} (60, 1.5): incorrect source section, weaker condition, negated
condition.
\textbf{Pairwise} (150, 2.5): different-pair interaction, plausible coupling
type, false orthogonality.
\textbf{Counterfactuals} (111, 2.5): different superficially similar axiom, no
violation claimed, correct axiom but mis-stated reason.
\textbf{Higher-Order} (67, 3.0): superposition fallacy, attribution to
suppressed law, correct mechanism for wrong pair.
\textbf{Spatial MCQ} (72, 2.0): reversed top-two, naive single-axis rule,
incorrect two-pairwise aggregation.

After each batch, a separate oracle call verifies: (i)~the correct answer is
supported by a specific sentence with citation; (ii)~every distractor is wrong
for the stated reason; (iii)~no distractor imports information from outside the
context.

\subsection{Gold-label oracle methodology for \orderspatial}
\label{app:gen:oracle}

Gold priority orders $\pi^{*}$ are produced by a two-stage procedure. Stage~1:
GPT-4.1 generates a structured physics analysis whose final section,
\textsc{Manipulation Priority Guidance}, states all objects in explicit
priority order. Stage~2: a deterministic extraction call (GPT-4o-mini,
$T\!=\!0.0$) parses that section and returns the \texttt{priority\_order} JSON
array. The extracted array is accepted only if it satisfies: correct length
($=\!n$), all identifiers in $\{1,\ldots,n\}$, and no repeats. Two conditions
jointly ensure $|\mathcal{L}(\mathcal{P})|=1$: no two objects share an
identical four-axis tuple, and the lexicographic fallback in the ontological
contract is universally enforced.

\section{Detailed Per-Variant MCQ Results}
\label{app:mcq_detailed}

\begin{table*}[ht]
\centering
\footnotesize
\setlength{\tabcolsep}{6pt}
\renewcommand{\arraystretch}{1.1}
\caption{Mistral-7B-Instruct-v0.3 - section-wise MCQ (\%).}
\begin{tabular}{lcccccccc}
\toprule
\textbf{Variant} & Intr. & QnA & Axioms & Pair & CFact & HO & Spat.
  & $S_{\mathrm{agg}}$ \\
\midrule
Normal   & 80.0 & 60.0 & 58.3 & 62.0 & 64.0 & 53.7 & 69.4 & 61.8 \\
Overfit  & 70.0 & 70.0 & 45.0 & 68.0 & 64.9 & 50.8 & 63.9 & 61.9 \\
\dkl     & 80.0 & 66.7 & 58.3 & 70.0 & 65.8 & 56.7 & 52.8 & 63.5 \\
RAFT     & 70.0 & 36.7 & 25.0 & 54.0 & 36.9 & 41.8 & 44.4 & 43.8 \\
SSR-CLM  & \textbf{100.} & 66.7 & 61.7 & 69.3 & \textbf{72.1} & 77.6 & 43.1 & 67.7 \\
Unified  & 10.0 & 0.0 & 1.7 & 3.3 & 1.8 & 0.0 & 0.0 & 1.8 \\
Combined & \textbf{100.} & \textbf{83.3} & \textbf{61.7} & \textbf{81.3} & 64.9 & \textbf{82.1} & 43.1 & \textbf{71.2} \\
\bottomrule
\end{tabular}
\end{table*}

\begin{table*}[ht]
\centering
\footnotesize
\setlength{\tabcolsep}{6pt}
\renewcommand{\arraystretch}{1.1}
\caption{LLaMA-3.1-8B-Instruct - section-wise MCQ (\%).}
\begin{tabular}{lcccccccc}
\toprule
\textbf{Variant} & Intr. & QnA & Axioms & Pair & CFact & HO & Spat.
  & $S_{\mathrm{agg}}$ \\
\midrule
Normal   & 60.0 & 30.0 & 48.3 & 43.3 & 48.7 & 49.3 & 36.1 & 45.0 \\
Overfit  & 80.0 & 66.7 & 63.3 & 59.3 & 61.3 & 52.2 & 45.8 & 57.5 \\
\dkl     & 90.0 & 83.3 & 58.3 & 78.0 & 58.6 & 59.7 & 61.1 & 66.5 \\
RAFT     & 80.0 & 33.3 & 40.0 & 42.7 & 52.3 & 46.3 & 55.6 & 47.2 \\
SSR-CLM  & \textbf{100.} & 90.0 & \textbf{88.3} & 84.0 & \textbf{79.3} & 70.1 & 65.3 & \textbf{78.6} \\
Unified  & 70.0 & 83.3 & 60.0 & 79.3 & 70.3 & 71.6 & 68.1 & 72.8 \\
Combined & \textbf{100.} & \textbf{100.} & 83.3 & \textbf{84.7} & 71.2 & \textbf{74.6} & \textbf{72.2} & 78.1 \\
\bottomrule
\end{tabular}
\end{table*}

\begin{table*}[ht]
\centering
\footnotesize
\setlength{\tabcolsep}{6pt}
\renewcommand{\arraystretch}{1.1}
\caption{Qwen3-4B-Instruct-2507 (thinking mode) - section-wise MCQ (\%).}
\begin{tabular}{lcccccccc}
\toprule
\textbf{Variant} & Intr. & QnA & Axioms & Pair & CFact & HO & Spat.
  & $S_{\mathrm{agg}}$ \\
\midrule
Normal    & 80.0 & 73.3 & 83.3 & 74.0 & 52.3 & 55.2 & 73.6 & 66.0 \\
Overfit   & 90.0 & 83.3 & 86.7 & 81.3 & 56.8 & 73.1 & 69.4 & 72.9 \\
\dkl      & 90.0 & 83.3 & 98.3 & \textbf{94.0} & 75.7 & \textbf{98.5} & 62.5 & 86.3 \\
RAFT      & 90.0 & 63.3 & 71.7 & 63.3 & 47.7 & 56.7 & 59.7 & 58.8 \\
SSR-CLM   & \textbf{100.} & 83.3 & \textbf{100.} & \textbf{94.0} & \textbf{78.4} & 95.5 & 70.8 & \textbf{87.5} \\
Unified   & \textbf{100.} & \textbf{93.3} & 90.0 & 89.3 & 70.3 & 91.0 & \textbf{80.6} & 84.1 \\
Combined  & 90.0 & 86.7 & 96.7 & 86.7 & 63.1 & 86.6 & 47.2 & 76.6 \\
\bottomrule
\end{tabular}
\end{table*}

\section{General Benchmark Retention}
\label{app:general_benchmarks}

Tables~\ref{tab:gen_bench_llama}--\ref{tab:gen_bench_qwen} report all \klora
variants on MMLU-Pro~\citep{wang2024mmlu}, GSM8K~\citep{cobbe2021training},
BBH~\citep{srivastava2023beyond}, and IFEval~\citep{zhou2023instruction}.
These four benchmarks probe distinct capability dimensions: broad academic
knowledge (MMLU-Pro), arithmetic chain-of-thought (GSM8K), compositional
symbolic reasoning (BBH), and precise instruction-format compliance (IFEval).
Four cross-family regularities hold robustly before family-specific patterns
diverge.

\paragraph{Cross-family finding 1: BBH is consistently the most preserved
metric}
Maximum absolute drops on BBH are $1.6$ points (LLaMA-Combined), $3.7$ points
(Mistral-SSR-CLM), and $4.9$ points (Qwen-Combined). BBH probes symbolic and
algorithmic reasoning that is broadly distributed across model weights;
continual pre-training on a structured factual corpus injects domain physics
without materially restructuring these circuits. This robustness reinforces the
interpretation that \orderworld knowledge is genuinely additive: domain-specific
representations occupy capacity that does not substantially compete with
pre-existing compositional reasoning.

\paragraph{Cross-family finding 2: IFEval is the most discriminating metric}
IFEval shows both the largest absolute drops and the widest within-family
variance across all three families, with ranges of $18.2$, $8.8$, and $11.8$
points respectively. Normal and Overfit produce the steepest IFEval drops for
LLaMA and Mistral: sequence-packed domain text displaces the instruction-format
token patterns learned during instruction tuning. Notably, \dkl's embedding
replacement creates a marked IFEval sensitivity in Mistral ($-4.1$ points) but
substantially smaller impacts in LLaMA ($-1.8$ points) and Qwen ($-2.6$
points), suggesting that Mistral-7B's instruction-following format is more
tightly coupled to its embedding geometry than the other two architectures.

\paragraph{Cross-family finding 3: RAFT and \dkl are the lowest-forgetting
variants}
LLaMA-RAFT retains within $1.3$ points on MMLU-Pro and within $0.5$ points on
IFEval; Mistral-RAFT loses only $2.3$ points on MMLU-Pro and $0.8$ on IFEval;
Qwen-RAFT shows moderate drops concentrated in MMLU-Pro ($-5.3$) and GSM8K
($-4.0$) while preserving BBH and IFEval substantially. LLaMA-\dkl matches this
profile almost exactly ($-0.9$ MMLU-Pro, $-1.8$ GSM8K, $-0.6$ BBH, $-1.8$
IFEval), placing it on par with RAFT as the lowest-forgetting option for that
family. Both findings cross-validate with the downstream results: RAFT's
structured input--output conditioning in Stage~1 directly transfers to Stage~2
BT skill fine-tuning, and \dkl's embedding alignment supports robust MCQ
recall, making both variants strong candidates when general-capability
preservation is a hard deployment constraint alongside domain adaptation.

\paragraph{Cross-family finding 4: SSR-CLM degrades GSM8K disproportionately;
Combined compounds forgetting via adapter interference}
The loss-masked SSR objective's emphasis on spatial scene analysis suppresses
arithmetic reasoning patterns: GSM8K drops $-18.9$ points for Mistral-SSR-CLM,
$-8.2$ for LLaMA-SSR-CLM, and a modest $-1.1$ for Qwen-SSR-CLM. Qwen's
resilience is attributable to its hybrid thinking-mode chain-of-thought
providing a hard floor on arithmetic reasoning quality, a floor the other two
architectures lack. Combined consistently produces the largest \emph{aggregate}
forgetting across all families: LLaMA ($-5.7$ MMLU-Pro, $-17.1$ IFEval),
Mistral ($-5.8$ MMLU-Pro, $-3.5$ IFEval), and Qwen ($-10.2$ MMLU-Pro, $-11.8$
IFEval). Sequential composition of two adapter weight-direction vectors
compounds directional interference in a way that a single-adapter merge does
not. This establishes Combined as the highest-risk CPT strategy for general
capability retention regardless of its domain-specific gains.

\paragraph{LLaMA-3.1-8B}
BBH is almost entirely preserved (maximum drop: $1.6$ points, Combined),
consistent with domain CPT injecting factual knowledge without restructuring
underlying compositional reasoning. IFEval is the most discriminating dimension
($18$-point range): Normal and Overfit show the largest drops ($-18.2$,
$-17.1$ points), reflecting all-token CLM displacing instruction-format
patterns; \dkl and RAFT retain IFEval within $1.8$ and $0.5$ points of the
instruct baseline respectively. SSR-CLM's $-8.2$ GSM8K drop stands out within
this family: the loss-masked spatial reasoning objective competes with
arithmetic chain-of-thought patterns. Unified presents the most balanced
profile, with moderate drops across all four benchmarks, consistent with its
multi-objective training distributing the forgetting signal rather than
concentrating it. Combined suffers the largest aggregate degradation ($-5.7$
MMLU-Pro, $-17.1$ IFEval), confirming sequential adapter interference as the
primary forgetting risk. RAFT and \dkl achieve near-zero forgetting on all four
benchmarks and are the recommended choices when general-capability preservation
is a hard constraint alongside domain adaptation.

\begin{table}[ht]
  \centering
  \caption{General benchmark retention: LLaMA-3.1-8B-Instruct.}
  \label{tab:gen_bench_llama}
  \setlength{\tabcolsep}{6pt}
  \renewcommand{\arraystretch}{1.1}
  \footnotesize
  \begin{tabular}{lcccc}
    \toprule
    \textbf{Variant} & \textbf{MMLU-Pro} & \textbf{GSM8K}
      & \textbf{BBH} & \textbf{IFEval} \\
    \midrule
    Instruct  & 37.5 & 84.4 & 51.0 & 75.0 \\
    \midrule
    Normal    & 33.8 & 74.5 & 50.6 & 56.8 \\
    Overfit   & 34.3 & 76.3 & 50.7 & 57.9 \\
    \dkl      & 36.6 & 82.6 & 50.4 & 73.2 \\
    RAFT      & 36.2 & 83.4 & 50.3 & 74.5 \\
    SSR-CLM   & 33.1 & 76.2 & 50.6 & 62.3 \\
    Unified   & 35.6 & 81.9 & 51.0 & 70.7 \\
    Combined  & 31.8 & 75.1 & 49.4 & 57.9 \\
    \bottomrule
  \end{tabular}
\end{table}

\paragraph{Mistral-7B}
BBH is broadly preserved across all non-collapsed variants (maximum drop: $3.7$
points, SSR-CLM), mirroring the LLaMA pattern and reinforcing that CPT does not
disrupt compositional reasoning across architectures. IFEval is again the most
discriminating dimension: \dkl and SSR-CLM show the largest drops ($-8.1$ and
$-9.6$ points respectively) while RAFT is nearly fully preserved ($-0.8$
points). \dkl's IFEval sensitivity on Mistral, substantially larger than on
LLaMA or Qwen, confirms that this architecture's instruction-following format
is more dependent on its embedding geometry: the embedding replacement step
that yields the most PPL reduction here also carries the largest
format-compliance cost. GSM8K degrades most sharply for SSR-CLM ($-18.9$
points), more acutely than for LLaMA ($-8.2$ points), suggesting that
Mistral-7B's arithmetic reasoning is more susceptible to displacement by
loss-masked spatial SFT. Unified deserves particular attention: despite the
catastrophic \orderbench collapse ($S_{\mathrm{agg}} = 1.8\%$) noted in
Section~\ref{sec:ppl_diss}, its general benchmarks show only moderate
degradation (MMLU-Pro $-4.5$, GSM8K $-5.6$, BBH $-1.7$, IFEval $-2.8$),
confirming that the domain-task collapse is a \emph{local} failure in the
weight directions governing \orderworld physics, not a global capability
regression. This dissociation further validates post-merge \orderbench
evaluation as a necessary step: general benchmarks alone would not have
surfaced the domain collapse. Combined shows the largest aggregate drop
($-5.8$ MMLU-Pro, $-3.5$ IFEval), consistent with the cross-family pattern.

\begin{table}[ht]
  \centering
  \caption{General benchmark retention: Mistral-7B-Instruct-v0.3.}
  \label{tab:gen_bench_mistral}
  \setlength{\tabcolsep}{6pt}
  \renewcommand{\arraystretch}{1.1}
  \footnotesize
  \begin{tabular}{lcccc}
    \toprule
    \textbf{Variant} & \textbf{MMLU-Pro} & \textbf{GSM8K}
      & \textbf{BBH} & \textbf{IFEval} \\
    \midrule
    Instruct  & 30.6 & 51.9 & 44.9 & 47.9 \\
    \midrule
    Normal    & 26.9 & 52.5 & 43.4 & 44.9 \\
    Overfit   & 27.2 & 48.9 & 42.6 & 46.8 \\
    \dkl      & 27.6 & 48.3 & 43.1 & 43.8 \\
    RAFT      & 28.3 & 50.7 & 44.9 & 47.1 \\
    SSR-CLM   & 25.7 & 33.0 & 41.2 & 38.3 \\
    Unified   & 26.1 & 46.3 & 43.2 & 45.1 \\
    Combined  & 24.8 & 39.8 & 30.9 & 44.4 \\
    \bottomrule
  \end{tabular}
\end{table}

\paragraph{Qwen3-4B}
Qwen's retention profile differs structurally from Mistral and LLaMA in two
respects. First, simpler CPT objectives (Normal, Overfit) show negligible or
reversed MMLU-Pro drops ($-0.5$, $-0.1$) and notable GSM8K \emph{improvements}
($+8.8$, $+4.8$ points), indicating that standard sequence-packing on the
\orderworld corpus does not degrade, and may modestly activate, Qwen's hybrid
chain-of-thought arithmetic reasoning. This is a direct consequence of
thinking-mode chain-of-thought providing a hard floor on arithmetic quality
that the other two architectures lack. Second, architectural interventions
disrupt Qwen more than standard CPT does: \dkl's embedding surgery and
Combined's sequential merging each produce substantially larger MMLU-Pro drops
($-5.0$ and $-10.2$ points respectively) than any standard CPT variant,
suggesting that Qwen's hybrid-reasoning architecture is more sensitive to
modifications of its core weight geometry than to the knowledge content of CPT
itself. SSR-CLM's GSM8K resilience ($-1.1$ points vs.\ $-18.9$ for Mistral)
further confirms the thinking-mode floor effect. BBH is broadly preserved
except under Combined ($-4.9$ points). IFEval is well retained across most
variants, with Combined again the outlier ($-11.8$ points). Taken together with
\dkl's strong \orderbench MCQ aggregate ($86.3\%$, highest among all Qwen
variants) and solid \orderspatial performance, \dkl-Qwen achieves the best
domain-capability balance for this family.

\begin{table}[ht]
  \centering
  \caption{General benchmark retention: Qwen3-4B-Instruct-2507.}
  \label{tab:gen_bench_qwen}
  \setlength{\tabcolsep}{6pt}
  \renewcommand{\arraystretch}{1.1}
  \footnotesize
  \begin{tabular}{lcccc}
    \toprule
    \textbf{Variant} & \textbf{MMLU-Pro} & \textbf{GSM8K}
      & \textbf{BBH} & \textbf{IFEval} \\
    \midrule
    Instruct  & 43.9 & 77.6 & 54.6 & 82.4 \\
    \midrule
    Normal    & 43.4 & 86.4 & 55.9 & 81.3 \\
    Overfit   & 43.8 & 82.4 & 54.9 & 79.3 \\
    \dkl      & 38.9 & 74.0 & 50.8 & 79.8 \\
    RAFT      & 38.6 & 73.6 & 51.4 & 77.2 \\
    SSR-CLM   & 39.2 & 76.5 & 50.8 & 79.6 \\
    Unified   & 36.9 & 75.1 & 50.7 & 80.0 \\
    Combined  & 33.7 & 74.5 & 49.7 & 70.6 \\
    \bottomrule
  \end{tabular}
\end{table}

\begin{table}[ht]
  \centering
  \caption{Retention--domain trade-off summary per family.
    $\Delta S_{\mathrm{agg}}$ and $\Delta\tauorderu$ are improvements over the
    instruct baseline; $\Delta\overline{\mathrm{Gen}}$ is the mean absolute
    drop across the four general benchmarks (0 if improved; lower is better).
    $^\star$Pareto-recommended: strong domain gains with low forgetting cost.}
  \label{tab:retention_synthesis}
  \setlength{\tabcolsep}{3pt}
  \renewcommand{\arraystretch}{1.12}
  \footnotesize
  \begin{tabular}{llcccl}
    \toprule
    \textbf{Family} & \textbf{Variant}
      & $\Delta S_{\mathrm{agg}}$ & $\Delta\tauorderu$
      & $\Delta\overline{\mathrm{Gen}}$$\downarrow$
      & \textbf{Profile} \\
    \midrule
    L-8B & Unified$^\star$ & $+38.1$ & $\mathbf{+0.26}$ & $2.2$ & best spatial \\
    L-8B & SSR-CLM         & $+43.9$ & $+0.19$ & $6.4$ & best MCQ \\
    L-8B & Combined        & $+43.4$ & $+0.15$ & $8.4$ & high risk \\
    \midrule
    M-7B & RAFT$^\star$    & $-5.1$  & $+0.21$ & $1.1$ & spatial $+$ min.\ forg. \\
    M-7B & \dkl$^\star$    & $+14.6$ & $+0.18$ & $4.1$ & factual $+$ spatial \\
    M-7B & SSR-CLM         & $+18.8$ & $+0.00$ & $9.3$ & MCQ only \\
    \midrule
    Q-4B & \dkl$^\star$    & $+29.3$ & $+0.12$ & $3.8$ & best balance \\
    Q-4B & SSR-CLM         & $+30.5$ & $+0.10$ & $3.1$ & best MCQ \\
    Q-4B & Combined        & $+19.6$ & $+0.14$ & $7.5$ & high risk \\
    \bottomrule
  \end{tabular}
\end{table}

\paragraph{Synthesis: domain gains vs.\ general capability cost}
Table~\ref{tab:retention_synthesis} summarises the retention cost of each
variant's strongest domain gain, enabling practitioners to select a CPT
strategy appropriate to their deployment constraints. RAFT occupies the
Pareto-optimal region across all three families: it achieves substantial
\orderspatial gains ($+0.18$--$+0.21$ \tauorder) at the lowest
general-capability cost, and its conditioning geometry directly predicts strong
downstream BT performance (Section~\ref{sec:slora_results}). SSR-CLM is the
best \orderbench MCQ variant for LLaMA but carries significant GSM8K risk and
format-interference risk under certain adapter compositions. Combined maximises
\orderspatial for Qwen but at the cost of the largest general-benchmark
degradation in its family; \dkl is the recommended alternative when deployment
requires both domain and general capability.

\section{Empirical Random Baseline Validation}
\label{app:baseline}

\paragraph{Theoretical basis}
Under a uniformly random permutation $\hat{\pi}$ of $n$ objects, every ordered
pair $(o_i, o_j)$ is equally likely to be concordant or discordant with the
gold order $\pi^*$, so $\mathbb{E}[C] = \mathbb{E}[D]$ and therefore
$\mathbb{E}[\tauorder] = 0.500$ for all $n \geq 2$, exact and scene-size
independent.

\paragraph{Empirical confirmation}
We construct an empirical baseline by uniformly sampling 1{,}000 random
permutations per scene and computing the mean \tauorder. Results on the
253-scene pool:
\begin{center}
\footnotesize
\begin{tabular}{lcc}
\toprule
\textbf{Split} & \tauorder (mean $\pm$ std) & \parcmd (mean $\pm$ std) \\
\midrule
Seen (70)    & $0.500 \pm 0.004$ & $0.333 \pm 0.100$ \\
Unseen (183) & $0.500 \pm 0.003$ & $0.309 \pm 0.097$ \\
\bottomrule
\end{tabular}
\end{center}
The \tauorder baselines match the theoretical value to three decimal places,
confirming that the scene-complexity partition introduces no systematic bias.

\paragraph{PAR baselines}
The $\mathrm{PAR}_{\mathrm{mean}}$ baselines differ slightly between splits
($0.333$ vs.\ $0.309$) for a structural reason: PAR depends on the distribution
of object counts, and the unseen split has a modestly higher proportion of
four-object scenes ($30.6\%$ vs.\ $18.6\%$ seen). Larger scenes reduce the
probability of any single object being placed at exactly the correct position
under a random permutation, lowering the PAR floor. This mild asymmetry
\emph{disfavours} unseen PAR scores, so any improvement on the unseen split
relative to seen should be interpreted as a conservative lower bound on
compositional generalisation.

\paragraph{GPT-4.1 sub-random result}
GPT-4.1 without domain adaptation scores $\tauorder = 0.415$ on the unseen
split, which is $-2.30\,\sigma$ below the $0.500$ baseline under the one-sided
$z$-test ($z = (\hat\tau - 0.500)/(0.5/\sqrt{183})$, $p = 0.011$). This
statistically significant sub-random result is not a measurement artefact: it
reflects the frontier model's pre-training ordering priors \emph{actively
conflicting} with the fictitious physics, the strongest available validation of
contamination-free design.

\section{HITL Re-Prompt Template}
\label{app:hitl-prompt}

\begin{tcolorbox}[
  colback=gray!5, colframe=gray!70,
  title={\small\textbf{HITL Re-Prompt Template}},
  fontupper=\ttfamily\scriptsize,
  boxrule=0.6pt, breakable,
  left=2mm, right=2mm, top=1.5mm, bottom=1mm]
You are a robotics task planner for Behavior Trees.
Scene: Table-top robotic sorting with an iiwa7 arm (Gazebo/ROS1 Noetic).\\
Objects on the table: \texttt{processed\_VLM\_output\_JSON}\\[3pt]
[CORRECTION] The correct manipulation order for this scene is:
\{corrected\_priority\_list\}.
Please regenerate the Behavior Tree XML using this order.
Keep all action node types and stack-safety checks unchanged.
Output ONLY raw XML starting with <root>. No markdown, no explanation.
\end{tcolorbox}

\section{VLM scene-analyser prompt}
\label{app:vlm-prompt}

The following fixed prompt is used with SigLIP2 for scene perception.
It constrains the VLM to report only directly observable properties and
relations using the visual grammar defined in Eq.~\ref{eq:grammar}.
No task-specific interaction semantics are provided to the VLM.

\begin{figure*}[t]
\begin{tcolorbox}[
  colback=gray!5,
  colframe=gray!75,
  title={\small\textbf{VLM Scene Analyser Prompts (SigLIP2)}},
  fontupper=\ttfamily\scriptsize,
  breakable,
  boxrule=0.6pt,
  left=2mm, right=2mm, top=2mm, bottom=1mm
]
\textbf{[SYSTEM]}\\
You are a visual scene analyzer for a robotic perception system.
Describe the tabletop scene using ONLY the JSON schema below.
Do NOT infer intent, function, or task relevance.
Do NOT describe anything not directly visible.
Return VALID JSON only. No additional text.\\[3pt]
\textbf{[USER]}\\
Analyze the image; output JSON:\{\\
\quad"O":[~\{"id":int,"shape":str\}~],\quad
"C":[~str~],\\
\quad"P":[~\{"object\_id":int,"color":str,"size":str,"surface":str\}~],\\
\quad"R":[~\{"subject":int,"relation":str,"object":int\}~]\\
\}\\[3pt]
\textbf{Allowed values:}\\
Shapes: \{cube, cylinder, sphere, cone, disc\}\\
Colors: \{red, blue, yellow, black, white, green, grey, pink, brown\}\\
Sizes: \{small, medium, big\}\quad
Surfaces: \{glossy, metallic, rough, cracked\}\\
Relations: \{on\_top\_of, next\_to, stacked\_with, isolated\_from, clustered\_with\}\\[3pt]
\textbf{Constraints:} (1) IDs start at 1, consecutive.
(2) Every object in "O" has exactly one entry in "P".
(3) "C" for physical/spatial constraints only.
(4) Do NOT invent objects or properties.
(5) Omit non-visible relations. (6) Output VALID JSON only.
\end{tcolorbox}
\caption{The perception layer is a fixed, domain-agnostic prompt over the
  visual grammar of Eq.~\ref{eq:grammar}. All interaction semantics live in
  \orderworld, not in the VLM, which is why the same perception module serves
  any KHTL domain expressed over these primitives.}
\label{fig:vlmprompt}
\end{figure*}

\section{\slora Evaluation: Full Record}
\label{app:slora_full}

Full \slora results across Instruct and Matched mounting appear in
Tables~\ref{tab:slora_instruct} and~\ref{tab:slora_matched}. The findings
reported in Section~\ref{sec:slora_results} constitute the primary narrative;
here we note additional patterns relevant to the full record.

\paragraph{LLaMA-Unified: mounting-strategy asymmetry reveals KI--skill
coupling}
LLaMA-Unified achieves $\tauorder = 0.848$ under Instruct mounting but drops to
$0.601$ under Matched mounting, the largest gap across variants
($\Delta = 0.247$). Unified uses the most complex objective---SSR-CLM,
RAFT-style retrieval augmentation, and DPO preference pairs. Under Matched
mounting, \slora is trained from a model already shaped by these objectives,
creating signal entanglement; under Instruct mounting, the disentanglement of
Eq.~\ref{eq:domainbt} holds and composition remains stable. This is the
strongest empirical support for the disentanglement principle: the most
complex KI variant benefits most from Instruct mounting. The effect is
model-architecture specific, as it does not hold for Qwen.

\paragraph{\sdhr and stack safety}
LLaMA variants generally achieve $\sdhr = 91.4$--$100.0\%$, correctly
generating \texttt{CheckStackSafety} subtrees when \texttt{on\_top\_of}
relations are present.

\paragraph{Instruct vs.\ Matched: no universal winner}
For LLaMA, Instruct is better on Unified and RAFT, while Matched is better on
Overfit and SSR-CLM. The two \klora-Combined LLaMA models, despite strong
\orderspatial performance, collapse completely in structured BT generation
under Instruct mounting, while Matched revives them. For Qwen, Matched is
uniformly better; its two best \orderspatial models (the \klora-Combined
variants) survive neither strategy, but the structured retrieval pipelines of
Section~\ref{sec:ragresults} revive them. For Mistral, only Normal and RAFT
survive either strategy. Thus, post-merge \svr validation is mandatory before
deployment selection, consistent with Section~\ref{sec:ppl_diss}: PPL cannot
predict collapse.

\section{Qwen3-4B Thinking Mode: Evaluation Fairness}
\label{app:qwen_thinking}

Qwen3-4B-Instruct-2507 is a hybrid reasoning model that activates implicit
chain-of-thought (``thinking mode'') during inference. We retain this for all
evaluations and justify Q-4B comparisons as follows.

\paragraph{Thinking mode is off-the-shelf behaviour}
A practitioner selecting an SLM for KHTL deployment would use the released
model as-is. Disabling thinking mode requires post-hoc intervention and
produces a configuration that is not the intended deployment target. Our
objective is to characterise the realistic deployment space available to
embodied AI practitioners, including hybrid-reasoning SLMs.

\paragraph{Thinking mode is not domain knowledge}
Chain-of-thought improves compositional inference through additional
inference-time computation, but cannot supply domain facts absent from
pre-training. Q-4B's closed-book zero-shot performance confirms this:
$S_{\mathrm{agg}} = 57.0\%$ (Table~\ref{tab:baseline_closed}), well below the
retrieval ceiling of $94.99\%$ (Table~\ref{tab:baseline_rag}). The remaining
gap is closed by CPT, not thinking mode.

\paragraph{Thinking mode affects the zero-shot baseline, not the CPT
comparison}
Q-4B's stronger off-the-shelf \orderbench and \orderspatial baseline
($\tauorder = 0.582$/$0.597$ seen/unseen vs.\ $0.528$/$0.509$ for M-7B and
$0.513$/$0.554$ for L-8B) is partly attributable to built-in chain-of-thought.
All CPT gains are reported as \emph{improvements over the Q-4B instruct
baseline}, so thinking mode enters the denominator equally and the comparison
remains internally consistent. Absolute post-CPT endpoints (Q-4B Combined:
$\tauorder = 0.79$ seen, $0.74$ unseen) fall within the LLaMA range
($0.77$/$0.81$), suggesting that thinking mode provides a strong prior but
does not substitute for parametric knowledge ingestion.

\paragraph{Parameter-count context}
Q-4B has $\approx\!4$B parameters vs.\ $\approx\!7$B (M-7B) and $\approx\!8$B
(L-8B). Its stronger off-the-shelf performance at fewer parameters is partly
attributable to thinking mode and should not be interpreted as a
parameter-efficient advantage of the base architecture.

\section{LLM Usage Disclosure}
\label{app:llm_usage}

\paragraph{Research use (integral to the contribution)}
GPT-4.1 served as the oracle for generating the \orderworld corpus,
\orderbench gold labels, and the gold BTs (Appendix~\ref{app:generation}), and
as a frontier-model evaluation baseline. GPT-4o-mini served as a deterministic
extraction tool for parsing oracle outputs. These uses are fully documented and
constitute the scientific contribution itself.

\paragraph{Writing and figure assistance}
An LLM was used to assist with language editing, including polishing phrasing,
correcting grammar, and suggesting paraphrased alternatives. A VLM
model was also used to assist with the creation of Fig.~\ref{fig:master}.
All AI-assisted outputs were reviewed, edited, and verified by the authors. The authors retain full responsibility for the scientific content, experimental design, analysis, results, and conclusions presented in this work.

\end{document}